\documentclass[twoside,11pt]{article}
\usepackage{comment} 
\usepackage{blindtext}

\usepackage{subcaption}

\usepackage{jmlr2e}

\usepackage{xcolor}

\usepackage{physics}
\usepackage{bbm}
\usepackage{amssymb}

\usepackage{graphicx}

\usepackage{algorithm}
\usepackage{algpseudocode}

\usepackage{hyperref}

\usepackage{multirow}
\usepackage{booktabs}
\usepackage{cellspace} 
\usepackage{threeparttable} 
\usepackage{longtable}
\usepackage{threeparttablex}

\usepackage{adjustbox}
\hypersetup{ hidelinks }

\usepackage{booktabs,tabularx,mathtools}

\newcommand{\Db}{\boldsymbol{D}}
\newcommand{\Lb}{\boldsymbol{L}}
\newcommand{\Ab}{\boldsymbol{A}}
\newcommand{\Ib}{\boldsymbol{I}}
\newcommand{\Zb}{\boldsymbol{0}}
\newcommand{\Jb}{\boldsymbol{J}}
\newcommand{\R}{\mathbb{R}}

\newcommand{\Ll}{\mathcal{L}}
\newcommand{\allC}{\{C_j\}_{j=1:n}}

\newcommand{\normt}[1]{\| #1 \|_2}
\newcommand{\diag}{\operatorname{diag}}

\newcommand{\N}{\mathbb{N}}

\allowdisplaybreaks

\usepackage{lastpage}
\jmlrheading{26}{2025}{1-\pageref{LastPage}}{11/25}{12/25}{00-0000}{Marius~F.~R.~Juston, Ramavarapu~S.~Sreenivas, Dustin~Nottage, Ahmet~Soylemezoglu}

\ShortHeadings{LDLT $\Ll$-Lipschitz Network}{Juston, Sreenivas et al.}
\firstpageno{1}

\begin{document}

\title{LDLT $\Ll$-Lipschitz Network: Generalized Deep End-To-End Lipschitz Network Construction}

\author{\name Marius~F.~R.~Juston \email mjuston2@illinois.edu \\
       \addr The Grainger College of Engineering, Industrial and Enterprise Systems Engineering Department \\
       University of Illinois Urbana-Champaign\\
       Urbana, IL 61801-3080, USA
       \AND
       \name Ramavarapu~S.~Sreenivas \email rsree@illinois.edu \\
       \addr The Grainger College of Engineering, Industrial and Enterprise Systems Engineering Department \\
       University of Illinois Urbana-Champaign\\
       Urbana, IL 61801-3080, USA
       \AND
       \name Dustin~Nottage \email dustin.s.nottage@erdc.dren.mil \\
       \addr Construction Engineering Research Laboratory \\
        U.S. Army Corps of Engineers Engineering Research and Development Center\\
       Urbana, IL 61822, USA
       \AND
       \name Ahmet~Soylemezoglu \email ahmet.soylemezoglu@erdc.dren.mil \\
       \addr Construction Engineering Research Laboratory \\
        U.S. Army Corps of Engineers Engineering Research and Development Center\\
       Urbana, IL 61822, USA}

\editor{My editor}

\maketitle

\begin{abstract}
Deep residual networks (ResNets) have demonstrated outstanding success in computer vision tasks, attributed to their ability to maintain gradient flow through deep architectures. Simultaneously, controlling the Lipschitz constant in neural networks has emerged as an essential area of research to enhance adversarial robustness and network certifiability. This paper presents a rigorous approach to the general design of $\Ll$-Lipschitz deep residual networks using a Linear Matrix Inequality (LMI) framework. Initially, the ResNet architecture was reformulated as a cyclic tridiagonal LMI, and closed-form constraints on network parameters were derived to ensure $\Ll$-Lipschitz continuity; however, using a new $LDL^\top$ decomposition approach for certifying LMI feasibility, we extend the construction of $\Ll$-Lipchitz networks to any other nonlinear architecture. Our contributions include a provable parameterization methodology for constructing Lipschitz-constrained residual networks and other hierarchical architectures. Cholesky decomposition is also used for efficient parameterization. These findings enable robust network designs applicable to adversarial robustness, certified training, and control systems. The $LDL^\top$ formulation is shown to be a tight relaxation of the SDP-based network, maintaining full expressiveness and achieving 3\%-13\% accuracy gains over SLL Layers on 121 UCI data sets.
\end{abstract}


\section{Introduction}

The robustness of deep neural networks (DNNs) is a critical challenge, particularly in safety-sensitive domains where small adversarial perturbations can lead to dangerous outcomes, such as the misclassification of important objects. One approach to address this issue is by enforcing Lipschitz constraints on the network architectures. These constraints ensure that small changes in the input do not significantly alter the output. This property is vital for certifying robustness against adversarial attacks, which involve introducing slight noise to modify the expected classification output result \citep{Inkawhich2019, Goodfellow2014}. The Lipschitz constant is a key measure to bound the network's sensitivity to input perturbations. Specifically, a $\Ll$-Lipschitz network can be theoretically guaranteed to remain stable within a defined ``stability sphere'' around each input, making it resistant to adversarial attacks up to a certain magnitude \citep{Tsuzuku2018}.

To achieve this, several methods have been proposed to enforce Lipschitz constraints on neural networks, including spectral normalization \citep{Miyato2018, Bartlett2017}, orthogonal parameterization \citep{Prach2022}, and more recent approaches such as Convex Potential Layers (CPL) and Almost-Orthogonal Layers (AOL) \citep{Meunier2022, Prach2022}. Previous works have been formulated under a unifying semidefinite programming architecture that imposes constraints on the networks as LMIs \citep{Araujo2023}. From the SDP programming architecture, the SDP-based Lipchitz layers (SLL) from \citet{Araujo2023} and most recently the Sandwich layers \citep{Wang2023DirectNetworks}, monotone Deep Equilibrium Models (DEQ) (MonDEQ) \citep{Winston2020MonotoneNetworks}, and Lipschitz-bounded equilibrium models (LBEN) \citep{Havens2023ExploitingModels} have emerged. \citet{Havens2023ExploitingModels} further shows how various Lipschitz structures (CPL, SLL, AOL, and Sandwich) interact in the DEQ setting. Our $LDL^\top$ construction could yield additional structures that DEQs can exploit, generalizing the methodology. However, ensuring Lipschitz constraints in deep architectures, particularly residual networks (ResNets), presents unique challenges due to their recursive structure. While prior work has made strides in constraining individual layers \citep{Araujo2023, Meunier2021} and in developing a unifying semidefinite programming approach, the generalized deep residual network formulation presents issues due to the pseudo-tri-diagonal structure of the LMI it imposes. Our formulation can be viewed as an extension of the unified SDP-based perspective of \citet{Araujo2023} to deep residual modules and multi-layered sub-systems, using block $LDL^\top$ factorization of the resulting LMI. 

Furthermore, multi-layered general Feedforward Neural Networks (FNN) have been shown to generate block tri-diagonal matrix LMI formulations \citep{Xu2024, Wang2023DirectNetworks} due to their inherent network structure, which, in contrast to the residual formulation, yield explicit solutions \citep{Sandryhaila2013, Agarwal2019}. However, due to the residual network's off-diagonal structure, applying the exact eigenvalue computation directly is not feasible, making the solution process significantly more complex.

Previous work has also demonstrated an iterative approach by using projected gradient descent or a regularization term on the estimated Lipschitz constant to enforce a constraint on the Lipschitz constant \citep{Gouk2021, Aziznejad2020, Bear2024}. While this guarantees an iterative enforcement of the Lipschitz constraint, it does not ensure a theoretical Lipschitz guarantee across the entire network until this convergence. However, the advantage of this technique is its generalizability, which allows for the usage of more general network structures.

\subsection{Contributions}

This paper introduces the formulation of deep residual networks as Linear Matrix Inequalities (LMI). It derives closed-form constraints on network parameters to ensure theoretical $\Ll$-Lipschitz constraints. The LMI was structured as a tri-diagonal matrix with off-diagonal components, which inherently complicates the derivation of closed-form eigenvalue computations. 

Moreover, while \citet{Araujo2023}'s work generates a closed-form solution for a residual network, it is limited to considering a single inner layer. In contrast, this paper presents a more general formulation that accommodates a more expressive inner-layer system within the residual network, offering greater flexibility and broader applicability.

This method decomposes the LMI system into an $LDL^\top$ formulation and restricts the diagonal matrix $D$ to be strictly positive definite. 
Formally, we show in Sections \ref{sec:block_decomp} to \ref{sec:ldlt_parameterization} that for fixed architecture and Lipschitz constant $\Ll$, the feasible set of $LDL^\top$-parameterized weights coincides with the feasible set of the SDP problem from \citet{Fazlyab2019}, restricted to general hierarchical neural network architectures. It is illustrated specifically in this paper  for the cyclic block-tri-diagonal LMI (up to measure-zero degeneracies). This implies that our parameterization is a tight reparameterization of the SDP-based bound in the sense of \citet{Wang2023DirectNetworks, Araujo2023}, rather than formulating the network as a composition of 1-Lipschitz layers. This allows the system to be viewed as an end-to-end parameterization and generalizes the architectures of the 1-Lipchitz formulation from \citet{Fazlyab2019, Araujo2023, Wang2023DirectNetworks}. Providing a standard decomposition methodology for any compositional or hierarchical non-linear system.

Efficient computation optimizations were derived for the $LDL^\top$ parameterization by converting the square root of the matrix operation, which is extremely expensive for large-dimensional matrices, into Cholesky decomposed triangular matrices.

Empirical comparisons with the current state of the art are provided on tabular classification benchmarks, 121 curated UCI data sets, with certified accuracies at different levels across algorithms statistically compared.

\section{LMI Formulation}

Following the work of \citet{Araujo2023}, who formulated a Lipschitz neural network as a constrained LMI problem to construct a residual network, limitations in their approach were identified. Specifically, their formulation yielded a single-layer residual network, which is inherently less expressive than the generalized deep-layered residual network popularized by architectures such as ResNet and its variants \citep{He2015ResNet, Szegedy2016, Zagoruyko2016, Hu2017, Xie2016}. These deeper networks perform better because the multiple inner layers within the modules enable more complex latent-space transformations, thereby increasing the network's expressiveness. This research focuses on establishing constraints for the inner layers to maintain the $\Ll$-Lipschitz condition while maximizing the expressiveness of the residual network for larger inner layers. As such, the inner layers of the residual network were represented as a recursive system of linear equations:
\begin{gather*}
    x_{k + 1} = A_k x_k + B_k w_{k,n},  \\
    v_{k,n} = C_n w_{k,n - 1} + b_n,  \\
    w_{k,n} =  \sigma_n(v_{k, n} ),  \\
     \vdots  \\
    v_{k,1} = C_1 x_k + b_1,  \\
    w_{k,1} =  \sigma_1(v_{k, 1} ). 
\end{gather*}
Where each of the layer parameters were defined as $C_l \in \mathbb{R}^{d_l \times d_{l - 1}}, b_l \in \mathbb{R}^{d_l}$ for $l \in \{1,\cdots,n\}$. When $n= 1$, the formulation reduces to the one presented in \citet{Araujo2023}, rendering it redundant in its derivation. The goal of the LMI was to maintain the Lipschitz constraint formulated as $\norm{x'_{k + 1} - x_{k + 1}} \leq \Ll\norm{x'_{k} - x_{k}}$.

Given that this system could be represented as a large recursive system, it was possible to split all the constraints of the inner layers as a set of LMI conditions similar to \citet{Araujo2023, Xu2024, 10.5555/3454287.3455312}. 

We only require that each non-linear function, in this case the activation functions, satisfy a pointwise slope restriction, which can be encoded via incremental quadratic constraints (IQC) as in  \citet{Xu2024, Araujo2023, Lessard2015AnalysisConstraints, Fazlyab2019, Megretski1997SystemConstraints}. Strong convexity is not required; in particular, ReLU and leaky ReLU fit in this framework with $m_i = 0$ and $L_i = 1$. The full parameterization of available activation functions can be found in Appendix \ref{sec:activation_quad_bounds}. The IQC is defined as:
\begin{align*}
\begin{bmatrix}
        v_k - v'_k \\
        w'_{k, i} - w_{k, i}
    \end{bmatrix}^\top  
    \begin{bmatrix}
        -2 L_i m_i \Lambda_i  & (m_i + L_i)\Lambda_i \\
         (m_i + L_i)\Lambda_i & -2 \Lambda_i
    \end{bmatrix}
    \begin{bmatrix}
        v_k - v'_k  \\
        w'_{k, i} - w_{k,i}
    \end{bmatrix} \leq 0, 
\end{align*}
where $\Lambda_n$ must be a positive definite diagonal matrix. Given that $v_k - v'_k = C_n (w_{k, n - 1}  - w'_{k, n - 1})$ the inequality thus becomes the following quadratic constraints, where $\Delta w_{k, i}$ was defined as $\Delta w_{k, i} =  w'_{k, i} - w_{k,i}, \forall i \in \{1, 2, \ldots, n\}$,
%
\begin{gather*}
        \begin{bmatrix}
        C_i \left(\Delta w_{k, i-1} \right) \\
        \Delta w_{k, i}
    \end{bmatrix}^\top  \begin{bmatrix}
        -2 L_i m_i \Lambda_i  & (m_i + L_i)\Lambda_i \\
         (m_i + L_i)\Lambda_i & -2 \Lambda_i
    \end{bmatrix}\begin{bmatrix}
        C_i \left(\Delta w_{k, i -1 } \right) \\
        \Delta w_{k, i}
    \end{bmatrix} \leq 0 . 
\end{gather*}
%
The following LMI could be formulated as the summation in Equation \eqref{eqn:lmi_formulation}.
\begingroup
\begin{small}
\setlength\arraycolsep{1pt}
\begin{align}
&\begin{bmatrix}
    x'_k - x_k \\
    w'_{k, 1} - w_{k, 1}\\
    w'_{k, 2} - w_{k, 2}\\
    \vdots \\
    w'_{k, n - 1} - w_{k, n - 1} \\
    w'_{k, n} - w_{k, n}
\end{bmatrix}^\top  \begin{bmatrix}
    \Ib_{d_x} & \Zb_{d_x}  \\
    \vdots & \Zb_{d_1} \\
    \vdots & \vdots \\
    \vdots & \vdots \\
    \Zb_{d_{n - 1}} & \vdots \\
    \Zb_{d_x} & \Ib_{d_x} 
\end{bmatrix} \begin{bmatrix}
    A_k^\top A_k - \Ll^2 \Ib & A_k^\top B_k \\
    B_k^\top A_k & B_k^\top B_k
\end{bmatrix}\begin{bmatrix}
    \Ib_{d_x} & \Zb_{d_x}  \\
    \vdots & \Zb_{d_1} \\
    \vdots & \vdots \\
    \vdots & \vdots \\
    \Zb_{d_{n - 1}} & \vdots \\
    \Zb_{d_x} & \Ib_{d_x} 
\end{bmatrix}^\top 
\begin{bmatrix}
    x'_k - x_k \\
    w'_{k, 1} - w_{k, 1}\\
    w'_{k, 2} - w_{k, 2}\\
    \vdots \\
    w'_{k, n - 1} - w_{k, n - 1} \\
    w'_{k, n} - w_{k, n}
\end{bmatrix}  +  \nonumber \\
\sum_{l = 1}^n 
&\begin{bmatrix}
    x'_k - x_k \\
    w'_{k, 1} - w_{k, 1}\\
    w'_{k, 2} - w_{k, 2}\\
    \vdots \\
    w'_{k, n - 1} - w_{k, n - 1} \\
    w'_{k, n} - w_{k, n}
\end{bmatrix}^\top E_i^\top  \begin{bmatrix}
    C_l & \Zb \\
    \Zb & \Ib
\end{bmatrix}^\top   \begin{bmatrix}
        -2 L_l m_l \Lambda_l  & (m_l + L_l)\Lambda_l \\
         (m_l + L_l)\Lambda_l & -2 \Lambda_l
    \end{bmatrix}\begin{bmatrix}
    C_l & \Zb \\
    \Zb & \Ib
\end{bmatrix}
E_i
\begin{bmatrix}
    x'_k - x_k \\
    w'_{k, 1} - w_{k, 1}\\
    w'_{k, 2} - w_{k, 2}\\
    \vdots \\
    w'_{k, n - 1} - w_{k, n - 1} \\
    w'_{k, n} - w_{k, n}
\end{bmatrix} \leq 0,  \label{eqn:lmi_formulation}
\end{align} 
\end{small} 
\endgroup
where,
\begin{gather*}
    D_l = \sum_{i = 1}^l d_i,  \ 
    E_l : \{0, 1\}^{\left(d_{l} + d_{l - 1}\right) \times D_n}, \mbox{and} \  
    [E_l]_{ij} = \begin{cases}
    1, & \text{if $j - D_l = i$}  \\
    0, & \text{else}
\end{cases}, 
\end{gather*}
moreover, $i$ and $j$ represented the row and column indices, respectively. The $E_l$ matrix represented a ``selection'' vector to ensure that the proper variables were used for the parameterization, which gave the following LMI in Equation \eqref{eqn:explicit_lmi}.
\begingroup
\setlength\arraycolsep{1pt}
\begin{equation}
\label{eqn:explicit_lmi}
\resizebox{0.97\textwidth}{!}{$
    \begin{bmatrix}
        A^\top A - \Ll^2 \Ib - 2 L_1 m_1 C_{1}^\top  \Lambda_{1} C_{1} 
            & (L_1 + m_1) C_{1}^\top  \Lambda_{1}                           & \Zb      & \Zb        & \Zb                                                         &  A^\top B                      \\
       (L_1 + m_1)   \Lambda_{1} C_{1}                                      
            & - 2 L_2 m_2 C_{2}^\top  \Lambda_{2} C_{2} -2  \Lambda_{1}    & \ddots & \Zb        & \Zb                                                         &  \Zb                         \\
        \Zb                                                       & \ddots                                                & \ddots & \ddots   & \Zb                                                         &  \Zb                         \\
        \Zb                                                       & \Zb                                                     & \ddots & \ddots   & \ddots                                                    &  \Zb                         \\
        \Zb                                                       & \Zb                                                     & \Zb      & \ddots   &  - 2 L_n m_n C_{n}^\top  \Lambda_{n} C_{n} -2  \Lambda_{n - 1}  &  (L_n + m_n)  C_n^\top  \Lambda_n  \\
        B^\top A                                                   & \Zb                                                     & \Zb      & \Zb        &(L_n + m_n)  \Lambda_n C_n                                    &  B^\top B-2\Lambda_n         
    \end{bmatrix}
    \preceq \Zb
$}
\end{equation}
\endgroup
We seek a parameterization of $A, \{\Lambda_1, \cdots, \Lambda_n\}$,$\{C_1, \cdots, C_n\}$, and $B$ that ensures the LMI was indeed negative semidefinite to satisfy the Lipschitz constraint where ideally $\{C_1, \cdots, C_n\}$ would be as unconstrained as possible to ensure expressive inner layers. From the LMI, 
it was discovered that explicitly deriving the network’s constraint based on its eigenvalues proved to be an exceptionally intricate task.

\section{LDL Decomposition}

The sparse block structure of the LMI allows the exploration of potential decompositions that would lend to a more straightforward solution. Given that we know that Equation \eqref{eqn:explicit_lmi} should be constructed as a positive semidefinite symmetric matrix, we can thus find a decomposition that would allow for a nice block structure, which would make it feasible to bound the eigenvalue signs quickly. Such a candidate is the Cholesky $LDL^\top$ decomposition, also known as the real square-root-free Cholesky decomposition. 
\begin{theorem} \label{th:ldlt_decomposition}
     Given a real symmetric positive-definite matrix, the factorization may be written as
     \begin{align}
        \Ab = \Lb \Db \Lb^\top, \nonumber 
     \end{align}
     where $\boldsymbol{L}$ is a lower unit triangular (unitrangular) matrix, and $\boldsymbol{D}$ is a (block) diagonal matrix. If $\boldsymbol{A}$ is positive definite, then $\boldsymbol{D}$ will also be positive definite \citep{Watkins2002}.
     Where the following recursive relations apply for the entries of $\Db$ and $\Lb$:
\begin{align}
    \Db_j &= \Ab_{jj} - \sum_{k = 1}^{j - 1} \Lb_{jk} \Db_k \Lb_{jk}^\top, \tag{LD1}\label{eq:LD1}\\
    \Lb_{ij} &= \left(\Ab_{ij} - \sum_{k = 1}^{j - 1}\Lb_{ik}\Db_{k}\Lb_{jk}^\top\right)\Db^{-1}, & \text{for } i > j. \tag{LD2}\label{eq:LD2}
 \end{align}
 \end{theorem}
An advantage of the $LDL^\top$ decomposition over the traditional Cholesky decomposition, which \citet{Wang2023DirectNetworks} alluded to, is the absence of the need to compute the square root of components. This simplification facilitates the computation and derivation of the necessary components for this system and enables a more straightforward formulation for the block variant form.
Given this block factorization, we can examine the current LMI in Equation \eqref{eqn:explicit_lmi}, and decompose each of its elements. However, since the current LMI assumes negative semidefiniteness, we must ensure that we reverse the sign of the matrix as $M \preceq \Zb \Leftrightarrow -M \succeq \Zb$.

For our application, given that we only need the system's positive definiteness to be enforced, it is only necessary to constrain all the $\Db_i$ such that $\Db_i \succeq \Zb$ for all $i \in \mathbb{N}$. 

\subsection{Block Decompositions} \label{sec:block_decomp}

The LMI is defined as a block matrix of $(n + 1) \times (n + 1)$ blocks,
\begin{lemma} \label{lm:da_block}
    The result of $\Db_1$ is equal to the symmetric matrix,
    \begin{align}
        \Db_1 &=  \Ll^2 \Ib + 2 L_1 m_1 C_{1}^\top  \Lambda_{1} C_{1} - A^\top A, \nonumber 
    \end{align}
\end{lemma}
additionally,
\begin{lemma}  \label{lm:dj_block}
 The result of $\Db_j$  for $j \in \{2, \cdots, n\}$ is equal to the symmetric matrix,
    \begin{align}
       \Db_j =& 2 L_j m_j C_{j}^\top  \Lambda_{j} C_{j} + 2  \Lambda_{j} - (L_{j - 1} + m_{j - 1})^2 \Lambda_{j - 1} C_{j - 1}\Db_{j - 1}^{-1} C_{j - 1}^\top \Lambda_{j - 1},
    \end{align} \nonumber 
\end{lemma}
where the triangular terms are represented as,
\begin{lemma}
    The block triangular terms of $\Lb_{ij}$  for $j = \{1, \cdots, n -1\}$, where $\Lb_{jj} = \Ib$, are the following,
    \begin{align}
    \Lb_{(j + 1)j} &= -(L_j + m_j)   \Lambda_{j} C_{j}\Db_j^{-1},  \label{eqn:first_l_block}\\
    \Lb_{(j + 2)j} &= \Zb, \nonumber \\
    &\vdots\nonumber \\
    \Lb_{nj} &= \Zb ,\nonumber \\
    \Lb_{(n + 1)j} &= -\Jb_j  \Db_{j}^{-1} \label{eqn:last_l_block},
    \end{align}
    where we define $\Jb_j$ as the following:
\begin{align}
    \Jb_j &=  \underbrace{\left[\prod_{k = 1}^{j - 1}(L_k + m_k)\right]}_{\Upsilon_j }B^\top A \underbrace{\left[\prod_{k = 1}^{j - 1}\Db_{k}^{-1}C_k^\top \Lambda_{k}\right]}_{{\Gamma_j}}. \nonumber 
\end{align}
\end{lemma}
\begin{lemma}
    The last triangular off diagonal component not explicitly stated is $\Lb_{(n + 1)n}$, which is defined as,
    \begin{align}
        \Lb_{(n + 1)n} &= -\Jb_{n}  \Db_{n}^{-1}  - (L_n + m_n)\Lambda_{n}C_{n}\Db_{n}^{-1}. \nonumber 
    \end{align}
\end{lemma}
Finally the last block for the decomposition $\Db_{n+1}$ is defined as such,
\begin{lemma}  \label{lm:dn_1_block}
    The symmetric block diagonal $\Db_{n + 1}$ is of the form,
    \begin{align}
        \Db_{n + 1} =& 2\Lambda_n-B^\top B - B^\top A \left[\sum_{j = 1}^{n}\Upsilon_j^2 \Gamma_{j} \Db_{j}^{-1}\Gamma_{j}^\top \right]  A^\top B -  (L_n + m_n)^2\Lambda_{n}C_{n}\Db_{n}^{-1}C_{n}^\top \Lambda_{n}, \nonumber 
    \end{align}
\end{lemma}
\begin{proof}
Let $M$ denote the symmetric matrix appearing in \eqref{eqn:explicit_lmi} (annotated as $\bar M$), which we negate to achieve the positive definite constraint, $M = -\bar M$. Thus the block entries of $M$ (indexed $1,\dots,n+1$) are the negated blocks of \eqref{eqn:explicit_lmi}. We seek an $\mathrm{LDL}^\top$ factorization
\begin{align*}
M = \Lb\,\Db\,\Lb^\top, 
\end{align*}
where $\Lb$ is block unit lower-triangular ($\Lb_{jj}=\Ib$) and $\Db=\diag(\Db_1,\dots,\Db_{n+1})$ is block diagonal. We prove by induction on $j$ that the block diagonal factors $\Db_j$ and the nonzero sub-diagonal blocks $\Lb_{ij}$ satisfy the formulas given in Lemmas \ref{lm:da_block}---\ref{lm:dn_1_block}.

\medskip\noindent\textbf{Base case ($j=1$).}  We start by recalling the standard recursive relations from Theorem \ref{th:ldlt_decomposition}. Applying \eqref{eq:LD1} returns,
\begin{align*}
\Db_1 = M_{11}.
\end{align*}
where $M_{11}$ from \eqref{eqn:explicit_lmi} outputs,
\begin{align*}
\Db_1 = \Ll^2 \Ib + 2 L_1 m_1 C_{1}^\top \Lambda_{1} C_{1} - A^\top A, 
\end{align*}
which matches Lemma \ref{lm:da_block}.

For the first column of $\Lb$-blocks, we use \eqref{eq:LD2} with $j=1$ from which we obtain that,
\begin{align*}
\Lb_{21} = M_{21}\Db_1^{-1}.
\end{align*}
Since $M_{21} = -(L_1 + m_1)\Lambda_1 C_1$, we have
\begin{align*}
\Lb_{21} = - (L_1 + m_1)\Lambda_1 C_1 \Db_1^{-1},
\end{align*}
which is the stated formula for $\Lb_{(1+1)1}$ from \eqref{eqn:first_l_block}. For $i = \{3, \dots ,n \}$ the matrix $M_{i1} =\Zb$, thus $\Lb_{i1}=\Zb$. Finally $M_{(n+1),1} = -B^\top A$, so
\begin{align*}
\Lb_{(n+1),1} = -B^\top A \Db_1^{-1}.
\end{align*}
Defining the empty products, $\Upsilon_1 := 1$ and $\Gamma_1 := \Ib$, we note $\Jb_1 = \Upsilon_1 B^\top A \Gamma_1 = B^\top A$, so $\Lb_{(n+1),1} = -\Jb_1 \Db_1^{-1}$ as claimed in \eqref{eqn:last_l_block}. This completes the base case.

\medskip\noindent\textbf{Inductive hypothesis.}  Fix $j$ with $1\le j\le n$. Assume for all $k$ with $1\le k\le j$ the diagonal blocks $\Db_k$ and the nonzero sub-diagonal blocks $\Lb_{ik}$ (for $i>k$) satisfy the Lemmas  \ref{lm:da_block}---\ref{lm:dn_1_block}:
We will prove the same pattern holds at index $j+1$ (for $j+1\le n+1$).

\medskip\noindent\textbf{Inductive step: diagonal block $\Db_{j+1}$ for $1\le j < n$.}
Apply \eqref{eq:LD1} with index $j+1$:
\begin{align*}
\Db_{j+1} = M_{j+1,j+1} - \sum_{k=1}^{j} \Lb_{\,j+1,k}\,\Db_k\,\Lb_{\,j+1,k}^\top.
\end{align*}
From \eqref{eqn:explicit_lmi} we read
\begin{align*}
M_{j+1,j+1} = 2 L_{j+1} m_{j+1} C_{j+1}^\top \Lambda_{j+1} C_{j+1} + 2\Lambda_{j}.
\end{align*}
Due to the sparsity of $M$ and the inductive hypothesis, the only nonzero $\Lb_{j+1,k}$ with $1\le k\le j$ is $\Lb_{j+1,j} = -(L_j + m_j)\Lambda_j C_j \Db_j^{-1}$. Which reduces the system to,
\begin{align*}
\Lb_{\,j+1,j}\Db_j\Lb_{\,j+1,j}^\top
= (L_j + m_j)^2 \Lambda_j C_j \Db_j^{-1} C_j^\top \Lambda_j.
\end{align*}
Thus,
\begin{align*}
\Db_{j + 1} =& 2 L_{j + 1} m_j C_{j + 1}^\top  \Lambda_{j + 1} C_{j + 1} + 2  \Lambda_{j + 1} - (L_{j} + m_{j})^2 \Lambda_{j} C_{j}\Db_{j}^{-1} C_{j}^\top \Lambda_{j},
\end{align*}
which matches Lemma \ref{lm:dj_block} (with index shift $j\mapsto j+1$).

\medskip\noindent\textbf{Inductive step: sub-diagonal blocks in column $j+1$.}
Use \eqref{eq:LD2}:
\begin{align*}
\Lb_{i,j+1} = \bigl(M_{i,j+1} - \sum_{k=1}^{j} \Lb_{ik}\Db_k\Lb_{j+1,k}^\top\bigr)\Db_{j+1}^{-1},\qquad i>j+1.
\end{align*}
For $i=\{j+3,\dots,n\}$ we have $M_{i,j+1}=\Zb$ and by the inductive hypothesis $\Lb_{ik}=\Zb$ for all $k\le j$; hence the internal components vanishes and $\Lb_{i,j+1}=\Zb$ for $i=\{j+3,\dots,n\}$.

For the last row $i=n+1$ we calculate,
\begin{align*}
\Lb_{(n+1),\,j+1} = \Bigl(M_{(n+1),\,j+1} - \sum_{k=1}^{j} \Lb_{(n+1),k}\Db_k\Lb_{j+1,k}^\top\Bigr)\Db_{j+1}^{-1}.
\end{align*}
By the inductive hypothesis $\Lb_{(n+1),k} = -\Jb_k \Db_k^{-1}$ for $k\le j$, and the only nonzero $\Lb_{j+1,k}$ with $k\le j$ is $\Lb_{j+1,j} = -(L_j + m_j)\Lambda_j C_j \Db_j^{-1}$. Therefore, the summation reduces,
\begin{align*}
\Lb_{(n+1),j}\Db_j\Lb_{j+1,j}^\top &= \bigl(-\Jb_j \Db_j^{-1}\bigr)\Db_j\bigl(-(L_j + m_j)\Lambda_j C_j \Db_j^{-1}\bigr)^\top \\
&= \Jb_j \Db_j^{-1} (L_j + m_j)\Lambda_j C_j^\top \Lambda_j.
\end{align*}
Substituting $M_{(n+1),j+1}$ from \eqref{eqn:explicit_lmi},
\begin{align*}
\Lb_{(n+1),\,j+1} = -\Jb_{j+1}\Db_{j+1}^{-1},
\end{align*}
which is exactly the formula claimed in \eqref{eqn:last_l_block}.

The immediate sub-diagonal $\Lb_{(j+1),j}$ is obtained directly from \eqref{eq:LD2} at step $j$:
\begin{align*}
\Lb_{(j+1),j} = M_{(j+1),j}\Db_j^{-1} = - (L_j + m_j)\Lambda_j C_j \Db_j^{-1}.
\end{align*}

\medskip\noindent\textbf{Final diagonal block $\Db_{n+1}$.}
Applying \eqref{eq:LD1} at $j=n+1$ returns:
\begin{align*}
\Db_{n+1} = M_{n+1,n+1} - \sum_{k=1}^{n} \Lb_{\,n+1,k}\Db_k\Lb_{\,n+1,k}^\top.
\end{align*}
From \eqref{eqn:explicit_lmi} we have $M_{n+1,n+1} = 2\Lambda_n - B^\top B$ and using the inductive results from above, we derive that,
\begin{align*}
\Lb_{\,n+1,k} &= -\Jb_k \Db_k^{-1}\quad (k=1,\dots,n-1),\qquad \\
\Lb_{\,n+1,n} &= -\Jb_n \Db_n^{-1} - (L_n + m_n)\Lambda_n C_n \Db_n^{-1},
\end{align*}
and telescoping the summation, while canceling $\Db_k$ with $\Db_k^{-1}$, produces the closed form,
\begin{align*}
    \Db_{n + 1} =& 2\Lambda_n-B^\top B - \Lb_{(n + 1)1}\Db_1\Lb_{(n + 1)1}^\top -  \Lb_{(n + 1)2}\Db_2\Lb_{(n + 1)2}^\top - \cdots - \Lb_{(n + 1)n}\Db_n\Lb_{(n + 1)n}^\top, \nonumber \\
    %
    %
    =& 2\Lambda_n-B^\top B -  \left(\sum_{j = 1}^{n}\Jb_{j} \Db_{j}^{-1}\Jb_{j}^\top  +  (L_n + m_n)^2\Lambda_{n}C_{n}\Db_{n}^{-1}C_{n}^\top \Lambda_{n} \right) , \nonumber \\
    =& 2\Lambda_n-B^\top B - B^\top A \left[\sum_{j = 1}^{n} \Upsilon_j^2 \Gamma_{j} \Db_{j}^{-1}\Gamma_{j}^\top \right]  A^\top B -  (L_n + m_n)^2\Lambda_{n}C_{n}\Db_{n}^{-1}C_{n}^\top \Lambda_{n}, \\
\Db_{n+1}
&= 2\Lambda_n - B^\top B
    - B^\top A \Bigl[\sum_{j = 1}^{n}\Upsilon_j^2 \Gamma_{j} \Db_{j}^{-1}\Gamma_{j}^\top \Bigr] A^\top B  - (L_n + m_n)^2\Lambda_{n}C_{n}\Db_{n}^{-1}C_{n}^\top \Lambda_{n},
\end{align*}
This matches the final Lemma \ref{lm:dn_1_block}.

\medskip\noindent\textbf{Conclusion.} We have shown the base case, where $j=1$, and proved by induction that if the asserted formulas are valid up to index $j$ then they must hold at index $j+1$. By induction, the formulas for $\Db_1$, $\Db_j$, for $j=\{2,\dots,n\}$, $\Db_{n+1}$, and for the nonzero block-sub-diagonal entries of $\Lb$ hold for all the necessary indices. This thus completes the proof.
\end{proof}
%
Given the new structure, we can see the sparsity of the system decomposition resulting in the initial sparsity of the general LMI with $\Lb_{(j + 2)j} = \cdots = \Lb_{nj} = \Zb$, which significantly reduces the diagonal components' complexities to constrain each of the parameters in the system.

The following properties are demonstrated empirically with the unit triangular matrix $L$ structure in Figure \ref{fig:ldl_l} and the block diagonal matrix $D$ illustrated in Figure \ref{fig:ldl_d}.
\begin{figure}[htb]
  \centering
  \begin{subfigure}[t]{0.45\textwidth}
    \centering
    \includegraphics[width=\textwidth]{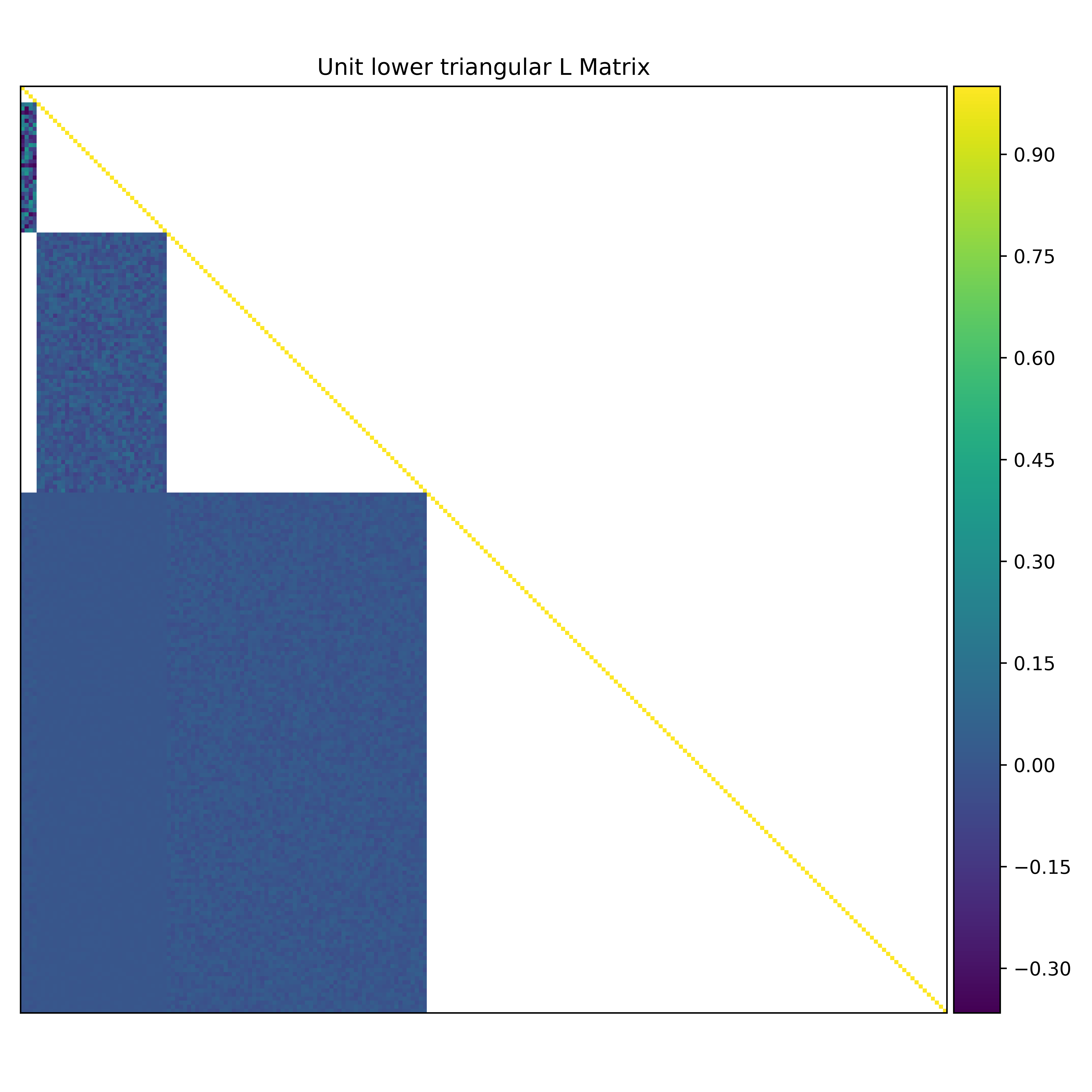}
    \caption{Unit lower–triangular factor \(\Lb\) (unitriangular), showing the sparsity pattern implied by the \(\mathrm{LDL}^\top\) block factorization of the negated LMI \(M=-\bar M\); in particular, \(\Lb_{(j+2)j}=\cdots=\Lb_{nj}=\mathbf 0\) and \(\Lb_{jj}=\Ib\).}
    \label{fig:ldl_l}
  \end{subfigure}
  \begin{subfigure}[t]{0.45\textwidth}
    \centering
    \includegraphics[width=\textwidth]{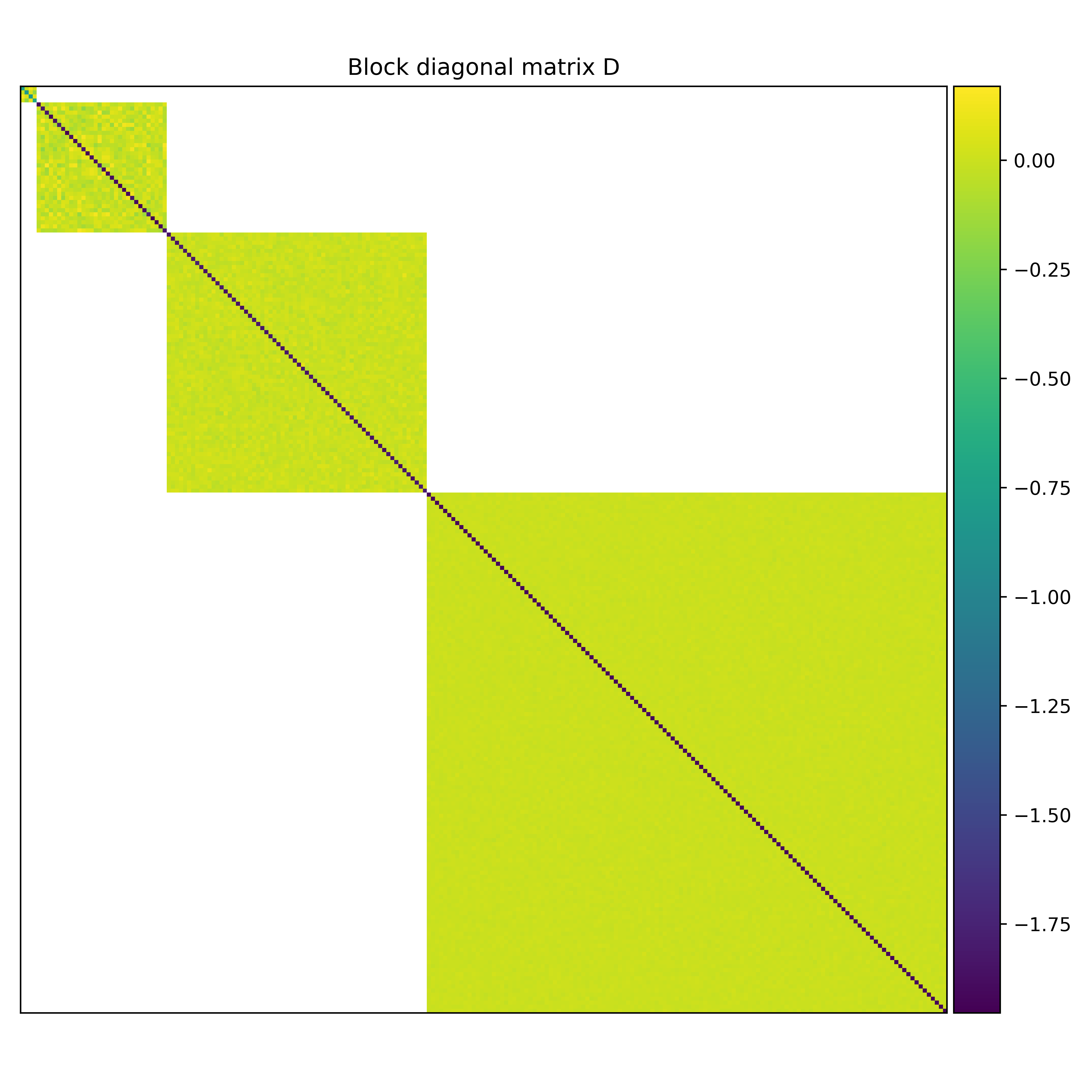}
    \caption{Block–diagonal factor \(\Db=\mathrm{diag}(\Db_1,\ldots,\Db_{n+1})\), where each block matches the closed forms in Lemmas~\ref{lm:da_block}–\ref{lm:dn_1_block} and satisfies \(\Db_i\succeq\mathbf 0\) under the imposed constraints.}
    \label{fig:ldl_d}
  \end{subfigure}

  \caption{Structure of the factors in the block \(\mathrm{LDL}^\top\) decomposition used to certify positive semidefiniteness of the LMI. Visualizations are generated from a randomly initialized \(\mathrm{LDL}^\top\) network with input/output dimension \(4\) and hidden layer widths \([32,\,64,\,256,\,256]\).}
  \label{fig:ldlt_overview}
\end{figure}
\subsection{Constraints} \label{sec:ldlt_constraints}

Given Table \ref{tab:activation_function_convecities} in Appendix \ref{sec:activation_quad_bounds}, the majority (56\%) of the activation function have the property that $P = 0$ and $S = 1$, for the continuation of this derivation we will assume that $L_i =1$ and $m_j = 0$, which thus generates the following symmetric block diagonals,
\begin{align*}
    \Db_1 &= \Ll^2 \Ib  - A^\top A , \\ 
     \Db_j &= 2  \Lambda_{j} - \Lambda_{j - 1} C_{j - 1}\Db_{j - 1}^{-1} C_{j - 1}^\top \Lambda_{j - 1}, \quad \forall j \in \{2, \cdots, n\}, \\
     \Db_{n + 1} &= 2\Lambda_n-B^\top \left(\Ib + A \left[\sum_{j = 1}^{n}\Gamma_{j} \Db_{j}^{-1}\Gamma_{j}^\top \right]  A^\top \right) B -  \Lambda_{n}C_{n}\Db_{n}^{-1}C_{n}^\top \Lambda_{n},
\end{align*}
for which, the following conditions are imposed on the matrices $D_{j} \succeq \Zb, \forall i \in \{1, \cdots, n + 1\}$ by finding the parameterization of $A, B, \{C_j, \Lambda_j\}_{j =1:n}$ such that the desired constraints are satisfied.
\begin{lemma} \label{lm:a_constraint}
    The matrix $A$ needs to be constrained such that,
    \begin{align*}
        \normt{A} \leq \Ll,
    \end{align*}
\end{lemma}
\begin{proof} \label{proof:a_block}
    Beginning with the first block,
    \begin{align*}
             \Db_1 &= \Ll^2 \Ib  - A^\top A \succeq 0. 
    \end{align*}
    which can thus be manipulated such that,
    \begin{align}
     A^\top A &\preceq \Ll^2 \Ib,
    \end{align}
    Given that $\normt{M}$ represents the spectral norm, defined $\normt{M} = \sqrt{\lambda_{\max}(M^\top M)}$, we have,
    \begin{align*} 
        \lambda_{\max} (A^\top A) &\leq \Ll^2, \nonumber \\
        \sqrt{\lambda_{\max} (A^\top A)} &\leq \Ll, \nonumber \\
          \normt{A} &\leq \Ll.
    \end{align*} 
\end{proof}
Based on the intermediary block $\Db_{j}$ for $j \in \{2, n\}$ we have that,
\begin{lemma}\label{lm:c_constraint}
    The matrices $\{C_{j} , \Lambda_j \}_{j = 1:n-1}$ need to be constrained such that,
    \begin{align*}
        \left\|C_{j}\Db_{j}^{-\frac{1}{2}}\right\|_2 \leq \sqrt{2}, \quad \forall j \in \{1, \cdots, n- 1\},
    \end{align*}
    with $\{\Lambda_j\}_{j = 1:n-1} = \Ib$
\end{lemma}
\begin{proof}
    Starting with the intermediary block definition for $j \in \{2, n\}$,
    \begin{align*}
        \Db_j &= 2  \Lambda_{j} - \Lambda_{j - 1} C_{j - 1}\Db_{j - 1}^{-1} C_{j - 1}^\top \Lambda_{j - 1} \succeq \Zb,
    \end{align*}
    we thus get the requirement that,
    \begin{align*}
        \Lambda_{j - 1} C_{j - 1}\Db_{j - 1}^{-1} C_{j - 1}^\top \Lambda_{j - 1} &\preceq 2  \Lambda_{j},
    \end{align*}
    where a simple solution can be derived if $\Lambda_{j - 1} = \Ib$,
    \begin{align}
        C_{j - 1}\Db_{j - 1}^{-1} C_{j - 1}^\top &\preceq 2 \Ib \label{eqn:cj_constraint}, \\
        \left(C_{j - 1}\Db_{j - 1}^{-\frac{1}{2}} \right) \left(\Db_{j - 1}^{-\frac{1}{2}} C_{j - 1}^\top \right) &\preceq 2 \Ib, \nonumber
    \end{align}
    where by the same argument in the proof \ref{proof:a_block}, the spectral norm is thus constructed. We can take the inverse square root of $D_{j -1}$, as we know by induction that $\Db_{j - 1} \succ \Zb$, which means that $\Db_{j - 1}^{-1} \succ \Zb$, which then means that there exists a square root of the matrix.
\end{proof}
The last two constraints that need to be obtained are the parameters $\{C_n, \Lambda_n\}$ and $B$, which can both be obtained through the last block $\Db_{n + 1}$,
\begin{lemma} \label{lm:c_b_constraint}
    The matrix $C_n$ has to be constrained such that,
    \begin{align*}
        \left\|C_{n}\Db_{n}^{-\frac{1}{2}}\right\|_2 \leq \sqrt{2},
    \end{align*}
    where $\Lambda_n = \Ib$ and $B$ needs to be constrained such that,
    \begin{align*}
        \left\|\left(\Ib + A \Sigma A^\top \right)^{1/2}B\right\|_2  \leq \sqrt{\left\|2\Ib - C_{n}\Db_{n}^{-1}C_{n}^\top \right\|_2}.
    \end{align*}
\end{lemma}
\begin{proof}
    Starting with the last block $\Db_{n + 1}$'s definition,
    \begin{align*}
        \Db_{n + 1} =& 2\Lambda_n-B^\top \left(\Ib + A \left[\sum_{j = 1}^{n}\Gamma_{j} \Db_{j}^{-1}\Gamma_{j}^\top \right]  A^\top \right) B -  \Lambda_{n}C_{n}\Db_{n}^{-1}C_{n}^\top \Lambda_{n},
    \end{align*}
    we set $\Lambda_n = \Ib$ resulting in,
    \begin{align*}
        \Db_{n + 1} =& 2\Ib-B^\top \left(\Ib + A \left[\sum_{j = 1}^{n}\Gamma_{j} \Db_{j}^{-1}\Gamma_{j}^\top \right]  A^\top \right) B -  C_{n}\Db_{n}^{-1}C_{n}^\top,
    \end{align*}
    we know by definition that all previous $\Db_j \succeq \Zb$ for $j \in \{1, \cdots, n\}$ based on the constraints defined in the lemmas above, as such $\Db_j^{-1} \succeq \Zb$ for $j \in \{1, \cdots, n\}$, which in turn means that $\Gamma_{j} \Db_{j}^{-1}\Gamma_{j}^\top \succeq \Zb$. and thus, $\Sigma = \sum_{j = 1}^{n}\Gamma_{j} \Db_{j}^{-1}\Gamma_{j}^\top  \succeq \Zb$. Thus,
    \begin{align*}
         B^\top\left(\Ib + A \Sigma A^\top \right) B  &\preceq 2\Ib - C_{n}\Db_{n}^{-1}C_{n}^\top.
    \end{align*}
    Given that $ B^\top\left(\Ib + A \Sigma A^\top \right) B \succeq \Zb$ this means that we need that $2\Ib - C_{n}\Db_{n}^{-1}C_{n}^\top  \succeq \Zb$ to provide a valid solution. As such 
\begin{align}
    2\Ib - C_{n}\Db_{n}^{-1}C_{n}^\top  \succeq \Zb \label{eqn:cn_constraint},
\end{align}
which is symmetric and diagonalizable,
which returns the constraints that
\begin{align}
    \left\|C_{n}\Db_{n}^{-\frac{1}{2}}\right\|_2 \leq \sqrt{2},
\end{align}
which results in
\begin{align}
     B^\top\left(\Ib + A \Sigma A^\top \right) B  \preceq 2\Ib - C_{n}\Db_{n}^{-1}C_{n}^\top, \label{eqn:b_constraint}  \\
     \left[B^\top\left(\Ib + A \Sigma A^\top \right)^{-\frac{1}{2}}\right] \left[\left(\Ib + A \Sigma A^\top \right)^{-\frac{1}{2}}  B\right]  \preceq M, \nonumber  \\
     \left\|B^\top\left(\Ib + A \Sigma A^\top \right)^{1/2}\right\|_2  \leq \sqrt{\normt{M}}. \nonumber
\end{align}
\begin{theorem}
    The spectral norm is invariant to the transpose, $\normt{A^\top A } = \normt{A A^\top} = \normt{A}^2$.
\end{theorem}
thus, applying it to the constraint yields the final constraint.
\end{proof}
The parameterization of $B$ can further be tightened in the following manner,
\begin{lemma}
    For the matrix $B$ and the positive definite symmetric matrices $C$ and $D$, the constraint
    \begin{align*}
        B^\top C B  \preceq D,
    \end{align*}
    can be constrained as,
    \begin{align*}
        B = C^{-\frac{1}{2}} W_{B}(\alpha_B \Ib + W_B^\top W_B)^{-\frac{1}{2}} D^{\frac{1}{2}},
    \end{align*}
\end{lemma}
\begin{proof}
    \begin{align}
     B^\top C B =&  D^{\frac{1}{2}} (\alpha_B \Ib + W_B^\top W_B)^{-\frac{1}{2}} W_{B}^\top C^{-\frac{1}{2}} C C^{-\frac{1}{2}} W_{B}(\alpha_B \Ib + W_B^\top W_B)^{-\frac{1}{2}} D^{\frac{1}{2}}, \nonumber \\
     =&  D^{\frac{1}{2}} (\alpha_B \Ib + W_B^\top W_B)^{-\frac{1}{2}} W_{B}^\top W_{B}(\alpha_B \Ib + W_B^\top W_B)^{-\frac{1}{2}} D^{\frac{1}{2}}, \nonumber \\
     \preceq& D^{\frac{1}{2}} \Ib D^{\frac{1}{2}} \preceq D.
    \end{align}
\end{proof}
this means that the most accurate constraint for the system,
\begin{align}
    B &= \left(\Ib + A \Sigma A^\top \right)^{-1/2} W_B(\alpha_B \Ib + W_B^\top W_B)^{-\frac{1}{2}} \left(2\Ib - C_{n}\Db_{n}^{-1}C_{n}^\top \right)^{\frac{1}{2}}.
\end{align}
\subsection{Variable Parameterization} \label{sec:ldlt_parameterization}

Now that each variable has constraints imposed on it, it is necessary to find a parameterization for each variable to ensure that each constraint is satisfied. 
\begin{lemma} \label{lm:weight_parameterization}
    If a matrix is parameterized as,
    \begin{align*}
        M = \gamma W(\alpha \Ib + W^\top W)^{-\frac{1}{2}},
    \end{align*}
    for any $W \in \R^{\dim(M)}$, and $\gamma, \alpha > 0$ (can be parameterized using $\gamma = e^{\bar{\gamma}}, \bar{\gamma} \in \R$). Then $\normt{M} \leq \gamma$
\end{lemma}
\begin{proof}
    Given the matrix $M$ parameterized as stated above,
    \begin{align*}
    M^\top M &= \gamma^2 (\alpha \Ib + W^\top W)^{-\frac{1}{2}} W^\top W (\alpha  \Ib + W^\top W)^{-\frac{1}{2}},  \\
            &=  \gamma^2  (\alpha \Ib + A)^{-\frac{1}{2}} A (\alpha \Ib + A)^{-\frac{1}{2}},
\end{align*}
where $A = W^\top W \succeq \Zb$. Since $A$ is a positive definitive matrix there exists an orthogonal matrix $Q$ ($Q^\top Q = \Ib$) and a diagonal matrix $\Lambda = \text{diag}(\mu_1,\cdots,\mu_n)$ with $\mu_i \ge 0$ such that $A = Q \Lambda Q^\top$,
\begin{align*}
    M^\top M &=\gamma^2  (\alpha \Ib + A)^{-\frac{1}{2}} A (\alpha \Ib + A)^{-\frac{1}{2}}, \\ 
    &= \gamma^2 Q (\alpha \Ib + \Lambda)^{-\frac{1}{2}}Q^\top (Q \Lambda Q^\top) Q (\alpha \Ib + \Lambda)^{-\frac{1}{2}}Q^\top,  \\
    &= \gamma^2 Q (\alpha \Ib + \Lambda)^{-\frac{1}{2}} \Lambda  (\alpha \Ib + \Lambda)^{-\frac{1}{2}}Q^\top,   \\
    &= \gamma^2 Q \text{diag}\left(\frac{\mu_1}{\alpha  + \mu_1}, \cdots, \frac{\mu_n}{\alpha  + \mu_n}\right)Q^\top,
\end{align*}
where $\frac{\mu_i}{\alpha + \mu_i} < 1$. As such $\normt{M}^2 = \max_i \frac{\mu_i}{\alpha+\mu_i}$. Since the largest $\mu_i$ is $\mu_\text{max} = \normt{W}^2$.

In turn:
\begin{align*}
    \normt{M}^2 &= \gamma^2  \frac{\mu_\text{max}}{\alpha + \mu_\text{max}} = \gamma^2 \frac{\normt{W}^2}{\alpha  + \normt{W}^2}, \\
    \normt{M} &= \gamma \frac{\normt{W}}{\sqrt{\alpha  + \normt{W}^2}} \leq \gamma.
\end{align*}
\end{proof}
Based on the parameterization in Lemma \ref{lm:weight_parameterization}, the parameterization of $A$ is derived as,
\begin{align}
    A = \Ll W_A(\alpha_A \Ib + W_A^\top W_A)^{-\frac{1}{2}}, \tag{PA} \label{eqn:a_formulation}
\end{align}
where $W_A \in \R^{\dim A}$, which guarantees the condition imposed in Lemma \ref{lm:a_constraint}. Similarly, $\allC$ are constructed as,
\begin{align}
    C_j = \sqrt{2} W_j(\alpha_j \Ib + W_j^\top W_j)^{-\frac{1}{2}} \Db_{j}^{\frac{1}{2}}, \quad \forall j \in \{1, \cdots, n\}, \tag{PC}\label{eqn:c_formulation}
\end{align}
where $W_j \in \R^{\dim C_j}$, $\forall j \in \{1, \cdots, n\}$ which also satisfies the constraints constructed in Lemma \ref{lm:c_constraint} and \ref{lm:c_b_constraint}. The final parameter $B$ can be parameterized as,
\begin{align} 
    B =& \left(\Ib + A \Sigma A^\top \right)^{-1/2} W_B(\alpha_B \Ib + W_B^\top W_B)^{-\frac{1}{2}} \left(2\Ib - C_{n}\Db_{n}^{-1}C_{n}^\top \right)^{\frac{1}{2}},  \label{eqn:b_formulation}
\end{align}
where $W_B \in \R^{\dim B}$, which guarantees the conditions in Lemma \ref{lm:c_b_constraint}. The recursive dependence on $\Db_{j}$, necessitates an eigen-decomposition of the matrix $\Db_{j}$, to compute the matrix square root. Both of these points will be addressed in Section \ref{sec:sqaure_root_comp}. 
This parameterization is closely related to spectral normalization \citep{Miyato2018}, and to a norm-based Lipschitz regularization \citep{Gouk2021}, but yields an extract spectral-norm bound rather than relying on power iteration or layer-wise upper bounds, acting closer to the orthonormal Cayley transform parameterization from \citet{Wang2023DirectNetworks}. 

\subsection{Diagonal Block Simplifications}

Given the new parameterizations of the parameters $A, B, \allC$, we can now derive simplifications for the blocks $\{D_j\}_{j = 1:n+1}$. Starting with $D_1$, we have that,
\begin{lemma}
    The simplification of $D_1$ is
    \begin{align*}
        D_1 = \Ll^2 \alpha_A (\alpha_A \Ib + W_A^\top W_A)^{-1}.
    \end{align*}
\end{lemma}
\begin{proof}
    We start with the initial definition of $D_1$, Lemma \ref{lm:da_block}, and $A$, defined in Equation \eqref{eqn:a_formulation},
    \begin{align*}
        D_1 =& \Ll^2 \Ib - A^\top A, \\
        =& \Ll^2 \Ib - \Ll^2 (\alpha_A \Ib + W_A^\top W_A)^{-\frac{1}{2}} W_A^\top W_A (\alpha_A \Ib + W_A^\top W_A)^{-\frac{1}{2}}  ,  \\
        =& \Ll^2 \left( \Ib -  (\alpha_A \Ib + M)^{-\frac{1}{2}} M (\alpha_A \Ib + M)^{-\frac{1}{2}} \right),  
    \end{align*}
    substituting $M =  (\alpha_A \Ib + M) - \alpha_A \Ib$, where $M = W_A^\top W_A $ and $V = (\alpha_A  \Ib + M)$,
    \begin{align*}
 \Db_1 &= \Ll^2 \left( \Ib -  V^{-\frac{1}{2}} (V - \alpha_A \Ib) V^{-\frac{1}{2}} \right),  \\
 &= \Ll^2 \left( \Ib -  \Ib + \alpha_A V^{-1} \right), \\
 &= \Ll^2 \alpha_A  \left( \alpha_A \Ib + M \right)^{-1}  =  \Ll^2 \alpha_A  \left( \alpha_A \Ib + W_A^\top W_A  \right)^{-1}. 
    \end{align*}
\end{proof}
The $\{D_{j}\}_{j=2:n}$ can also be simplified using,
\begin{theorem} \label{thm:woodbury_lemma}
    The Woodbury matrix identity \citep{Woodbury1950InvertingMatrices}, otherwise called the matrix inversion lemma, states that
    \begin{align*}
        (A + UCV)^{-1} = A^{-1} - A^{-1}  U(C^{-1} + V A^{-1}  U)^{-1}VA^{-1},  
    \end{align*}
    where $A \in \R^{n \times n}, U\in \R^{n \times k}, C\in \R^{k \times k}$ and $V\in \R^{k \times n}$ and $A$ is invertible.
\end{theorem} 
\begin{corollary} \label{lm:reduced_woodbury}
    The following matrix inverses are equivalent
    \begin{align*}
        \alpha (\alpha \Ib + MM^\top)^{-1} = \Ib -  M(\alpha \Ib + M^\top M)^{-1}M^\top,
    \end{align*}
    given through the Woodbury matrix identity in Theorem \ref{thm:woodbury_lemma}. By setting $A = \alpha \Ib$, $C = \frac{1}{\alpha}\Ib$, $U= \sqrt{\alpha} M$ and $V = \sqrt{\alpha}M^\top$
\end{corollary}
which allows for the follow simplifications of  $\{D_{j}\}_{j=2:n}$,
\begin{lemma} \label{lm:d_j_simpl}
    For  $j = \{2, \cdots, n\}$, $D_j$ can be simplified to be,
    \begin{align*}
        \Db_j &= 2\alpha_{j-1} (\alpha_{j-1}\Ib + W_{j - 1} W_{j - 1}^\top)^{-1}, \quad \forall j \in \{2, \cdots, n\}.
    \end{align*}
\end{lemma}
\begin{proof}
    Starting with the general definitions of $\{D_{j}\}_{j=2:n}$, Lemma \ref{lm:dj_block}, and $\allC$, defined in Equation \eqref{eqn:c_formulation}, we have that $V_j = \alpha_{j} \Ib + W_{j}^\top W_{j}$,
    \begin{align*}
         \Db_j  =& 2\Ib -  C_{j - 1}\Db_{j - 1}^{-1} C_{j - 1}^\top, \quad \forall j \in \{2, \cdots, n\},  \\ 
     =& 2 \Ib -  2 W_{j - 1}V^{-\frac{1}{2}} \Db_{j - 1}^{\frac{1}{2}}\Db_{j - 1}^{-1}\Db_{j - 1}^{\frac{1}{2}} V^{-\frac{1}{2}} W_{j - 1}^\top,   \\
     =& 2 \left(\Ib -  W_{j - 1}(\alpha_{j - 1}\Ib + W_{j - 1}^\top W_{j - 1})^{-1} W_{j - 1}^\top \right) .
    \end{align*}
    Using Corollary \ref{lm:reduced_woodbury} with $A = W_{j-1}$ we thus have that,
    \begin{align*}
        \Db_j &= 2\alpha_{j - 1} (\alpha_{j - 1}\Ib + W_{j - 1} W_{j - 1}^\top)^{-1}, \quad \forall j \in \{2, \cdots, n\}.
    \end{align*}
\end{proof}
Which thus simplifies the $\allC$ to,
\begin{align*}
    C_1 &=  \Ll \sqrt{2 \alpha_A} W_1\Omega_1^{-\frac{1}{2}} \Omega_0^{-\frac{1}{2}},  \\
    C_{j} &= 2 \sqrt{\alpha_{j-1}} W_j\Omega_j^{-\frac{1}{2}} \Omega_{j-1}^{-\frac{1}{2}}, \quad \forall j \in \{2, \cdots, n\},  
\end{align*}
with,
\begin{align*}
    \Omega_j = \begin{cases}
       \alpha_A \Ib + W_A^\top W_A, & \text{if } j = 0 \\
        \alpha_j \Ib + W_j W_j^\top, & \mathrm{else} \\
    \end{cases} .
\end{align*}
\subsection{B Expansion}

Based on the previous definition of $B$ in Equation \eqref{eqn:b_formulation} as,
\begin{align*}
    B =& \left(\Ib + A \Sigma A^\top \right)^{-1/2} W_B(\alpha_B \Ib + W_B^\top W_B)^{-\frac{1}{2}} \left(2\Ib - C_{n}\Db_{n}^{-1}C_{n}^\top \right)^{\frac{1}{2}},
\end{align*}
there are specific components that can be simplified in this system,

\subsubsection{Recursive Normalizer}

The recursive normalization constant, $\left(\Ib + A \Sigma A^\top \right)^{-1/2}$ needs to be expanded to be further understood now that interior components have been fully defined. The most important component in the system is the $A \Sigma A^\top $. We have,
\begin{align*}
    \Sigma = \sum_{j = 1}^{n}\Gamma_{j} \Db_{j}^{-1}\Gamma_{j}^\top, \qquad \Gamma_j = \prod_{k = 1}^{j - 1}\Db_{k}^{-1}C_k^\top,  
\end{align*}
which is very computationally inexpensive to compute when viewed as an iterative process. From the product form of $\Gamma_j$, we obtain the simple recurrence
\begin{align}
    \Gamma_{1} &= \Ib, \label{eq:Gamma_init}\\
    \Gamma_{j+1} &= \Gamma_{j} \Db_{j}^{-1} C_{j}^\top, \qquad j\ge 1. \label{eq:Gamma_recur}
\end{align}
which allows $\Sigma$ to be computed efficiently though Algorithm \ref{alg:iterative_sigma}, especially since if $\allC$ have been computed, the derivation of $D_j^{-1}$ is trivial since it equivalent to $c \Omega_{j -1}, c > 0$, which is much more computationally efficient than having to take the inverse of the matrices.
\begin{algorithm}[htb]
\caption{Iterative $\Sigma$ update}
\label{alg:iterative_sigma}
\begin{algorithmic}
\State \textbf{Input:} matrices $\{ \Db_j, C_j\}_{j=1}^n$ (each $\Db_j$ invertible)
\State \textbf{Initialize:} $\; \Gamma_1 \leftarrow \Ib,\; \Sigma \leftarrow \Zb$
\For{$j=1\ldots n$}
    \State $T \leftarrow \Gamma_j\,\Db_j^{-1}\,$ \Comment{compute once and reuse}
    \State $\Sigma \leftarrow \Sigma + T \Gamma_j^\top$ \Comment{since $T\Gamma^\top = \Gamma \Db_j^{-1} \Gamma^\top$}
    \State $\Gamma_{j+1} \leftarrow T C_j^\top$ \Comment{update to $\Gamma_{j+1}$}
\EndFor
\State \textbf{Output: } $\Sigma$
\end{algorithmic}
\end{algorithm}
we can simplify and expand the system to (setting all $\alpha = 1$),
\begin{align}
    \Gamma_j &= \prod_{k = 1}^{j - 1}\Db_{k}^{-1}C_k^\top = \prod_{k = 1}^{j - 1}\sqrt{2} \Db_{k}^{-1} \Db_{k}^{\frac{1}{2}}(\Ib + W_k^\top W_k)^{-\frac{1}{2}}W_k^\top =2^{\frac{j - 1}{2}} \prod_{k = 1}^{j - 1} \Db_{k}^{-\frac{1}{2}}(\Ib + W_k^\top W_k)^{-\frac{1}{2}}W_k^\top,  \nonumber \\
    &=2^{\frac{j - 1}{2}} \prod_{k = 1}^{j - 1} \left( \begin{cases}
        \Ll^2  & \text{k = 1} \nonumber \\
        2  & \mathrm{else}
    \end{cases}\Omega_{k - 1}^{-1} \right)^{-\frac{1}{2}}  (\Ib + W_k^\top W_k)^{-\frac{1}{2}}W_k^\top,  \nonumber  \\
    &=2^{\frac{j - 1}{2}} \prod_{k = 1}^{j - 1}  \begin{cases}
        \Ll^{-1}  & \text{k = 1} \nonumber \\
        2^{-\frac{1}{2}}  & \mathrm{else}
    \end{cases}\Omega_{k - 1}^{\frac{1}{2}} \Omega_k^{-\frac{1}{2}}W_k^\top, \\
    &= \begin{cases}
       1 & \text{$j = 1$} \\
       2^{\frac{j - 1}{2}} 2^{-\frac{j - 2}{2}} \Ll^{-1} & \mathrm{else} 
    \end{cases} \prod_{k = 1}^{j - 1} \Omega_{k - 1}^{\frac{1}{2}} \Omega_k^{-\frac{1}{2}}W_k^\top,  \nonumber \\
    &= \begin{cases}
       1 & \text{$j = 1$} \\
       \frac{\sqrt{2}}{\Ll} & \mathrm{else} 
    \end{cases} \underbrace{\prod_{k = 1}^{j - 1} \Omega_{k - 1}^{\frac{1}{2}} \Omega_k^{-\frac{1}{2}}W_k^\top}_{\Xi_j},  \label{eqn:reduced_gamma}
\end{align}
with,
\begin{align*}
    \Sigma &= \sum_{j = 1}^{n}\Gamma_{j} \Db_{j}^{-1}\Gamma_{j}^\top,  \\
    &= \sum_{j=1}^n \Gamma_{j} \left( \begin{cases}
        \Ll^2  & \text{j = 1}  \\
        2  & \mathrm{else}
    \end{cases}\Omega_{j - 1}^{-1} \right)^{-1} \Gamma_{j}^\top, \\
    &= \sum_{j=1}^n\left( \begin{cases}
        \Ll^{-2}  & \text{j = 1}  \\
        2^{-1}  & \mathrm{else}
    \end{cases} \right) \Gamma_{j} \Omega_{j - 1} \Gamma_{j}^\top, \\
    &= \sum_{j=1}^n\left( \begin{cases}
        \Ll^{-2} \times 1  & \text{j = 1}  \\
        2^{-1} \times \frac{2}{\Ll^2} & \mathrm{else}
    \end{cases} \right) \Xi_{j} \Omega_{j - 1} \Xi_{j}^\top ,\\
    &= \Ll^{-2} \sum_{j=1}^n  \Xi_{j} \Omega_{j - 1} \Xi_{j}^\top.
\end{align*}
The simplification of $A\Sigma A^\top$, helps further reduce the system,
\begin{align*}
    A  \Sigma A^\top =& \Ll^2 W_A \Omega_0^{-\frac{1}{2}} \Sigma  (\Omega_0^{-\frac{1}{2}})^\top W_A^\top,  \\
    =& W_A \Omega_0^{-\frac{1}{2}} \left( \sum_{j=1}^n  \Xi_{j} \Omega_{j - 1} \Xi_{j}^\top \right)  (\Omega_0^{-\frac{1}{2}})^\top W_A^\top,  \\
    =&  W_A W_A^\top  + W_A \Omega_0^{-\frac{1}{2}} \left( \sum_{j=2}^n  \Xi_{j} \Omega_{j - 1} \Xi_{j}^\top \right)  (\Omega_0^{-\frac{1}{2}})^\top W_A^\top. 
\end{align*}
Again we have the computational complexity of the problem being dominated from $\Omega_{k-1}^{\frac{1}{2}}$ and $\Omega_k^{-\frac{1}{2}}$ inside $\Gamma$, however, as from Section \ref{sec:sqaure_root_comp}, we can rewrite $\Omega_k$ by the decomposition,
\begin{align*}
    \Omega_k = R_k^\top R_k,
\end{align*}
which we already need to compute for the computational speedups, thus we have the very efficient computation of $\Omega_{k}^{\frac{1}{2}}$ simply as $R_k^\top$, and we would thus get equivalent singular value bounds,
\begin{lemma} \label{lm:s_r_link}
    Define $S = $
    For every PSD matrix $A$, we define $S = A^{\frac{1}{2}}$ as the unique symmetric PSD square root matrix ($A = SS^\top = S^\top S$) and $R$ the upper triangular Cholesky decomposition $A = R^\top R$. We state that there exists an orthogonal matrix $Q$ such that,
    \begin{align*}
        R^\top = S Q, \quad \text{equivalently} \quad S = R^\top Q^\top,
    \end{align*}
    and 
    \begin{align*}
        S^{-1} = Q R^{-\top}, \quad S^{-1} = Q^\top R^{-\top},
    \end{align*}
    depending on the chosen orientation for $Q$'s definition.
\end{lemma}
\begin{proof}
    Because $A = R^\top R = S^2 $, the matrices $R$ and $S$ are two invertible square roots of the same PSD matrix, thus the product 
    \begin{align*}
        Q = S^{-1}R^\top, \quad \text{or}, \quad   
        Q = R S^{-1},
    \end{align*}
    is well defined and thus satisfies,
    \begin{align*}
        Q Q^\top = S^{-1}R^\top R S^{-1} = S^{-1} A S^{-1} = S^{-1} S^2 S^{-1} = \Ib  .
    \end{align*}
\end{proof}
\begin{proposition}
The factor $A \Sigma A^\top$ cannot be rewritten as the computationally optimized version,
\begin{align}
    \bar{\Xi}_j &= \prod_{k = 1}^{j - 1} R_{k - 1}^\top R_k^{-\top}W_k^\top, \nonumber \\
     \bar{A}  \bar{\Sigma} \bar{A}^\top &=  W_A W_A^\top  + W_A R_A^{-1} \left( \sum_{j=2}^n  \Xi_{j} \Omega_{j - 1} \Xi_{j}^\top \right)  R_A^{-\top} W_A^\top. \label{eqn:simplified_sigma}
\end{align}
\end{proposition}
\begin{proof}
    To verify that this reparameterization of the system is indeed correct, we need to ensure that the norm of the matrix is maintained,
    such that $\normt{A \Sigma A^\top } = \|\bar{A} \bar{\Sigma} \bar{A}^\top \|$ and is equivalent between the two formulation.
    We start with the symmetric formulation where $S_j^2 = \Omega_j$.
    \begin{align*}
\Xi_j &= \prod_{k = 1}^{j - 1} S_{k - 1} S_k^{-1}W_k^\top,  \\
     A \Sigma A^\top &=  W_A W_A^\top  + W_A S_A^{-1} \left( \sum_{j=2}^n  \Xi_{j} \Omega_{j - 1} \Xi_{j}^\top \right)  S_A^{-\top} W_A^\top.         
    \end{align*}
    now we replace each $S$ with $S_j = R_j^\top Q_j^\top$ and $S_j^{-1} = Q_j^\top R_j^{-\top}$ , we get that,
    \begin{align*}
        S_{k - 1} S_j^{-1} &= (R_{j-1}^\top Q_{j-1}^\top)(  Q_j R_j^{-\top}),  \\
        &=R_{j-1}^\top ( Q_{j-1}^\top Q_j) R_j^{-\top},
    \end{align*}
    define the orthogonal factor,
    \begin{align*}
        \bar{Q}_j &= Q_{j-1}^\top Q_j,
    \end{align*}
    as such, it is improbable to have the decomposition be directly equal to each other such that the form proposed in Equation \eqref{eqn:simplified_sigma} is valid ($\bar Q_j = \Ib$ in that formulation). This would imply that $Q_{j-1}=   Q_j$, which by induction thus imposes the constraint that $Q_1 = \cdots = Q_n$. This would imply that $S_1 = \cdots = S_n$, which is not possible unless all $\Omega_j$ equal each other. This, by sheer randomness, is improbable and, if strictly imposed, would cause a major reduction in expressiveness.
    %
\end{proof}
\section{Computation Extensions}

In this Section, we first describe how we can efficiently compute the square root of a matrix using Cholesky decomposition rather than a costly full eigenvalue decomposition. Following this we also briefly expand on how to implement the work to convolutional neural networks.

\subsection{Square Root Computation} \label{sec:sqaure_root_comp}

When tasked with deriving the square root of a matrix, the first thought is for a matrix to satisfy the following conditions,
\begin{theorem}
    If $A$ is real symmetric positive semidefinite, then the square root, $A^{\frac{1}{2}}$, is used to denote the unique matrix $B$, that is positive semidefinite, and such that $BB= BB^T = A$ \citep{Higham1986NewtonsRoot}. There is precisely one square root $B$ that is both positive semidefinite and symmetric \cite[Theorem 7.2.6]{Horn2012MatrixAnalysis}.
\end{theorem}
This unique matrix is called the positive square root. To compute this is extremely computationally expensive, as it requires, naively, the use of an eigenvalue decomposition, such that,
\begin{align*}
    A = VQV^\top,
\end{align*}
where $A \in \R^{n\times n}$ is a symmetric positive semidefinite matrix, $Q$ is the diagonal eigenvalue matrix and $V$ the orthonormal eigenvectors matrix. The unique symmetric positive square root is thus, 
\begin{align*}
    B = VQ^{\frac{1}{2}}V^\top,
\end{align*}
however, to compute the eigenvalue problem, computation is an $\mathcal{O}(\frac{8}{3} n^3) + \mathcal{O}(n^2)$ algorithm \citep{Pan1999ComplexityEigenproblem}.

However, we do not necessarily need the symmetry constraint to be enforced in this situation. Instead we look at the Cholesky decomposition which is instead $\mathcal{O}(\frac{1}{3}n^3)+ \mathcal{O}(n^2)$ instead, which is much faster \citep{CholeskyLibrary}.

To validate this as a valid optimization, we need to ensure that the constraints are all still satisfied by this change. As such, we need to find parameterizations of $A, B, \allC$ that maintain their respective constraints.
\begin{lemma} \label{lm:a_simplification}
    We can define $A$ to be less computationally expensive by reformulating $A$ as,
    \begin{align*}
         A = \Ll W_A {R_A^{-1}},
    \end{align*}
    where $R_A$ is the upper Cholesky decomposition of $R_A^\top R_A = \alpha_A \Ib + W_A^\top W_A$, where $R_A$ is the upper triangular matrix.
\end{lemma}
\begin{proof}
    We start by finding the Cholesky decomposition of $(\Ib + W_A^\top W_A)^{-\frac{1}{2}}$, which must exist since the matrix is symmetric and positive definite, which implies that a Cholesky decomposition exists. It then follows that,
    \begin{align*}
       \alpha_A \Ib + W_A^\top W_A &= R_A^\top R_A, \\
       (\alpha_A  \Ib + W_A^\top W_A)^{-1} &= (R_A^\top R_A)^{-1} = R_A^{-1} R_A^{-\top}.
    \end{align*}
    Note that $R_A^{-1}$ is a Cholesky factor of the inverse, not the symmetric inverse square root (i.e. $R_A^{-1} \ne (\Ib + W_A^\top W_A)^{-\frac{1}{2}}$).
    We need to verify that this formulation satisfies the spectral norm constraint imposed on the system,
    \begin{align*}
        \Ll^2 &\ge \normt{A}^2= \lambda_{\max}(A^\top A) = \lambda_{\max}(A A^\top), \\
        &= \Ll^2 \lambda_{\max}( W_A {R_A^{-1}}R_A^{-\top}  W_A^\top ), \\
        &=  \Ll^2 \lambda_{\max}( W_A (\alpha_A  \Ib + W_A^\top W_A)^{-1} W_A^\top  ),
    \end{align*}
    which means that the spectral norm condition is still satisfied.
\end{proof}
Similarly, we have that,
\begin{lemma}
     We can define $\allC$ to be less computationally expensive by reformulating  $\allC$ as,
    \begin{align*}
         C_j = \sqrt{2} W_j R_j^{-1}  \begin{cases}
             \Ll \sqrt{\alpha_A} R_{A}^{-\top}, & j = 1 \\
             \sqrt{2 \alpha_{j- 1}}  R_{j-1}^{-\top}, & \mathrm{else}
         \end{cases},
    \end{align*}
    where $R_j$ is the upper triangular Cholesky decomposition of $R_j^\top R_j = \Omega_j$.
\end{lemma}
\begin{proof}
As a note, we need to satisfy the condition that,
\begin{align*}
    2\Ib &\succeq C_{j}\Db_{j}^{-1}C_{j}^\top,   
\end{align*}
from Equations \eqref{eqn:cj_constraint} and \eqref{eqn:cn_constraint}, where we thus decompose using Cholesky decomposition,
\begin{align*}
    D_j &= \begin{cases}
        \Ll^2 \alpha_A (\Ib + W_A^\top W_A)^{-1} = \Ll^2 \alpha_A R_A^{-1} R_A^{-\top},  & j = 1 \\
        2 \alpha_{j - 1} (\Ib + W_{j - 1} W_{j - 1}^\top)^{-1} =  2 \alpha_{j - 1} R_{j-1}^{-1} R_{j-1}^{-\top}, & \mathrm{else}
    \end{cases}.\nonumber
\end{align*}
For the case when $j = 1$, we have that,
\begin{align*}
     C_{1}\Db_{1}^{-1}C_{1}^\top &= 2 \Ll^2 \alpha_A W_1 R_1^{-1} R_A^{-\top} \left(\Ll^2 \alpha_A R_A^{-1} R_A^{-\top} \right)^{-1}R_{A}^{-1} R_1^{-\top}  W_1^\top,  \\
      &=2 W_1 R_1^{-1} (R_{A}^{-\top}R_{A}^{\top})( R_{A}R_{A}^{-1}) R_1^{-\top}  W_1^\top,  \\
      &= 2W_1 (R_1^{\top} R_1)^{-1}  W_1^\top , \\
      &= 2 W_1 (\alpha_1 \Ib + W_1 W_1^T)^{-1}  W_1^\top \preceq 2 \Ib.
      \end{align*}
Equivalently for the case when $j> 1$, we have that,
\begin{align*}
     C_{j}\Db_{j}^{-1}C_{j}^\top &= 4\alpha_{j-1} W_j R_j^{-1} R_{j-1}^{-\top}  \left(2 \alpha_{j-1} R_{j-1}^{-1} R_{j-1}^{-\top} \right)^{-1}R_{j-1}^{-1} R_j^{-\top}  W_j^\top,  \\
      &= 2 W_j R_j^{-1} (R_{j-1}^{-\top}R_{j-1}^{\top})( R_{j-1}R_{j-1}^{-1}) R_j^{-\top}  W_j^\top, 
      \end{align*}
following the same steps as for case $j = 1$, we have the same result.
\end{proof}
Finally, we have the simplification for $B$ as follows,
\begin{lemma}
     We can define $B$ to be less computationally expensive by reformulating  $B$ as,
    \begin{align*}
         B = c R_{\Sigma}^{-1}  W_B R_B^{-1},
    \end{align*}
    where $R_\Sigma$ and $R_B$ are the upper Cholesky decompositions of $\Ib + A \Sigma A^T$ and $\alpha_B\Ib + W_B^\top W_B$ respectively, where $R$ is an upper triangular matrix and $c =  \sqrt{\left\|2\alpha_n \left(\alpha_n \Ib + W_nW_n^\top \right)^{-1} \right\|_2}$.
\end{lemma}
\begin{proof}
From Equation \eqref{eqn:b_constraint} we need to satisfy the condition that,
\begin{align*}
     B^\top\left(\Ib + A \Sigma A^\top \right) B  \preceq 2\Ib - C_{n}\Db_{n}^{-1}C_{n}^\top ,  
\end{align*}
where we thus decompose using Cholesky decomposition,
\begin{align*}
     2\Ib - C_{n}\Db_{n}^{-1}C_{n}^\top &\succeq B^\top\left(\Ib + A \Sigma A^\top \right) B , \\
     &\succeq c^2 ( R_B^{-\top }   W_B^\top R_{\Sigma}^{-\top}) (R_{\Sigma}^\top R_\Sigma) (R_{\Sigma}^{-1}  W_B R_B^{-1}) , \\
     &\succeq c^2 R_B^{-\top }   W_B^\top (R_{\Sigma}^{-\top} R_{\Sigma}^\top ) ( R_\Sigma R_{\Sigma}^{-1})  W_B R_B^{-1},  \\
     &\succeq c^2 R_B^{-\top }   W_B^\top W_B R_B^{-1} .
\end{align*}
By construction we know that $c^2 \Ib \succeq 2\Ib - C_{n}\Db_{n}^{-1}C_{n}^\top$ as such,
\begin{align*}
    R_B^{-\top }   W_B^\top W_B R_B^{-1} &\preceq \Ib,   \\
     W_B^\top W_B &\preceq R_B^\top R_B,  \\
     W_B^\top W_B &\preceq \alpha_B\Ib + W_B^\top W_B.
\end{align*}
\end{proof}
as well,
\begin{lemma}
     We can define $B$ to be less computationally expensive by reformulating  $B$ as,
    \begin{align*}
         B = \sqrt{2 \alpha_n} R_{\Sigma}^{-1}  W_B R_B^{-1} R_C^{-\top},
    \end{align*}
    where $R_\Sigma$, $R_B$ and $R_C$ are the upper Cholesky decompositions of $\Ib + A \Sigma A^T$, $\alpha_B\Ib + W_B^\top W_B$ and $\Omega_n$ respectively, where $R$ is an upper triangular matrix.
\end{lemma}
\begin{proof}
From Equation \eqref{eqn:b_constraint} we need to satisfy the condition that,
\begin{align*}
     B^\top\left(\Ib + A \Sigma A^\top \right) B  \preceq 2\Ib - C_{n}\Db_{n}^{-1}C_{n}^\top , 
\end{align*}
where we thus decompose using Cholesky decomposition,
\begin{align*}
     =&B^\top\left(\Ib + A \Sigma A^\top \right) B , \\
      =&2 \alpha_n( R_C^{-1} R_B^{-\top }   W_B^\top R_{\Sigma}^{-\top}) (R_{\Sigma}^\top R_\Sigma) (R_{\Sigma}^{-1}  W_B R_B^{-1} R_C^{-\top}) , \\
     =&2 \alpha_nR_C^{-1} R_B^{-\top }   W_B^\top (R_{\Sigma}^{-\top} R_{\Sigma}^\top ) ( R_\Sigma R_{\Sigma}^{-1})  W_B R_B^{-1} R_C^{-\top},  \\
     =&2 \alpha_nR_C^{-1} R_B^{-\top }   W_B^\top W_B R_B^{-1} R_C^{-\top}  \\
     \preceq & 2 \alpha_n R_C^{-1} \Ib R_C^{-\top} = 2 \alpha_n \Omega_n^{-1} = 2 \alpha_n\left(\alpha_n \Ib + W_nW_n^\top \right)^{-1}
\end{align*}
\end{proof}
\begin{theorem}
    The inverse of an upper triangular matrix, if it exists, is upper triangular \citep{Taboga2021TriangularMatrix}.
\end{theorem}
\begin{theorem}
    The product of two upper triangular matrices is upper triangular \citep{Taboga2021TriangularMatrix}.
\end{theorem}
Considering that all operations employ triangular-triangular or triangular-dense matrix multiplication, it is evident that optimized multiplication routines capable of multiplying triangular matrices using more efficient computational methods, owing to their sparsity, can be implemented. However, PyTorch does not inherently possess this functionality, necessitating the creation of a custom CUDA kernel to implement this optimization. Although this minor optimization is insignificant compared to the substantial performance gains achieved by replacing the eigen-decomposition with the Cholesky decomposition function.
%
%
\begin{lemma} \label{lm:cholesky_expressiveness}
    Let $B \in \R^{n \times n}$ be symmetric positive definite. Let $R \in \R^{n \times n}$ be the (unique) Cholesky factor with positive diagonal such that $R^\top R = B$, and let $S = B^{\frac{1}{2}}$ be the (unique) symmetric positive definite square root of $B$, for any $m, n \in \mathbb{N}$,
    \begin{align*}
        \mathcal{H}_R &:= \{WR^{-1}: W \in \R^{m \times n}\},  
        \mathcal{H}_S := \{WS^{-1}: W \in \R^{m \times n}\}, 
    \end{align*}
    coincide $\mathcal{H}_R =  \mathcal{H}_S$, as such $W \rightarrow WR^{-1}$ and $W \rightarrow WS^{-1}$ have the same expressiveness.
\end{lemma}
\begin{proof}
    From Lemma \ref{lm:s_r_link}, we know that there exists an orthonormal matrix $Q$ such that $S^{-1} = R^{-1}Q$ and equivalently $R^{-1} = S^{-1}Q^\top$.
    \par
    Let $A \in \mathcal{H}_S$. Then $A = WS^{-1}$ for some $W \in \R^{m \times n}$, we can thus
    \begin{align*}
        A = W(R^{-1} Q) = (WQ)R^{-1},
    \end{align*}
    Since right-multiplication by the fixed invertible matrix $Q$ is a bijection on $\R^{m \times n}$, the product $WQ$ ranges over all $m \times n$ matrices as $W$, hence $A \in \mathcal{H}_R$.
    \par
    Similarly let $A \in \mathcal{H}_R$. Then $A = WR^{-1}$ for some $W \in \R^{m \times n}$, we can thus
    \begin{align*}
        A = W(S^{-1} Q^\top) = (W Q^{\top})S^{-1}.
    \end{align*}
    As before, right-multiplication by $Q^\top$ is a bijection on $\R^{m \times n}$, hence $A \in \mathcal{H}_S$.
    \par
    Thus, the two parameterizations have identical representational power; they differ only by an orthogonal factor in the learnable parameter $W$.
\end{proof}
From Lemma \ref{lm:cholesky_expressiveness}, we know that these less computationally expensive parameterizations of the parameters maintain the same level of expressiveness as their original parameterization.
\begin{corollary}
    The set of networks realizable by the Cholesky-based parameterization of $A, B, \allC$ is identical to that realizable by the symmetric square roots. Thus substituting the Cholesky factors for the symmetric ones, introduces no expressiveness loss, while greatly reducing the computation cost from $\mathcal{O}(\frac{8}{3}n^3)$ to $\mathcal{O}(\frac{1}{3}n)$.
\end{corollary}
\begin{figure}[htb]
    \centering
    \includegraphics[width=0.8\linewidth]{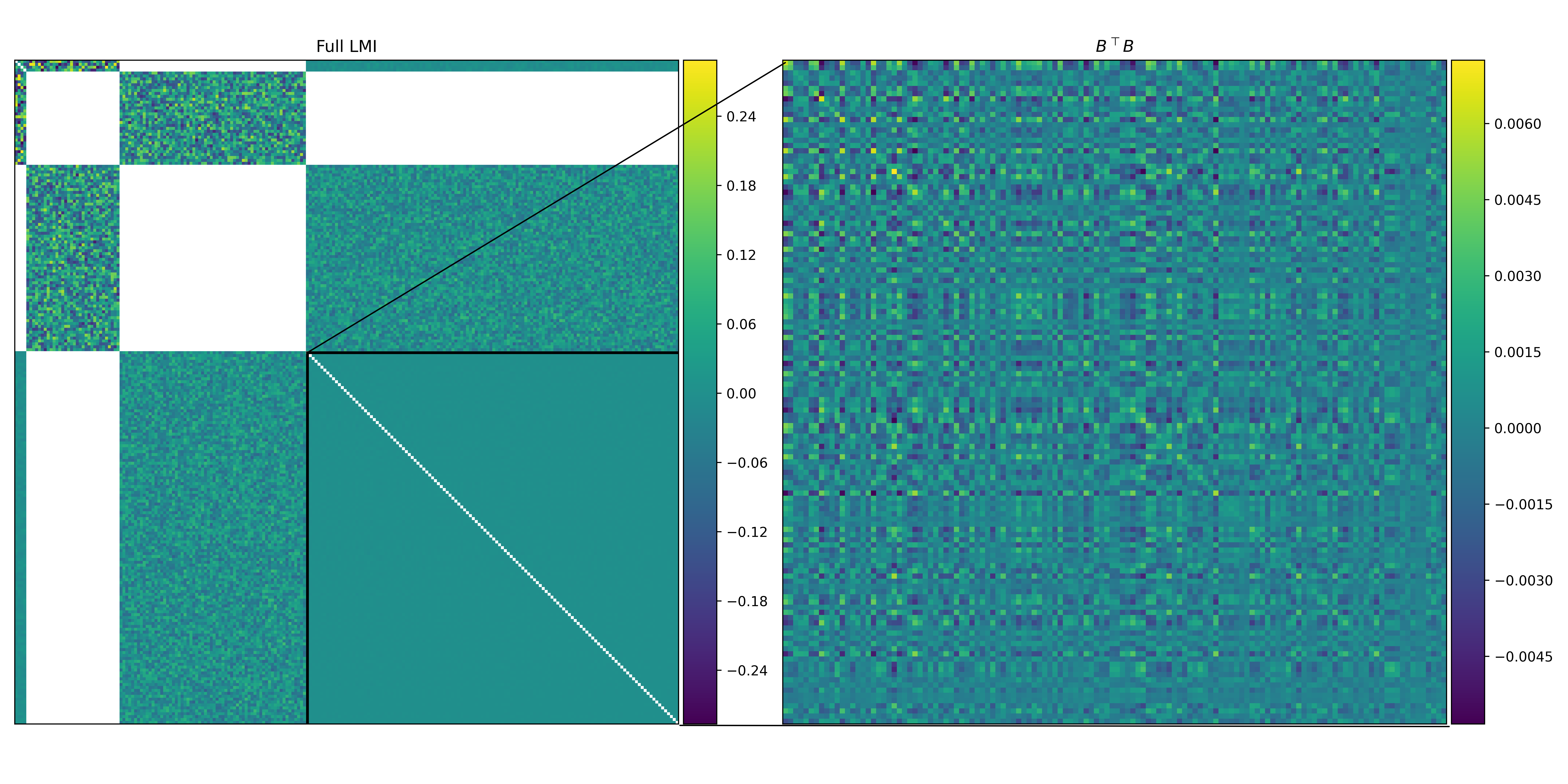}
    \caption{$B$ matrix artifacting in network architecture with input and output size of $4$ with layer widths of $[32, 64, 128]$ respectively}
    \label{fig:lmi_matrix_artifacting}
\end{figure}
\subsection{Convolution Layers}

Let us define $h, w, k \in \N$, and $n= wh$. We define the input $x \in \R^{h \times w \times c_I}$, where $c_I, c_O \in \N$, represents the number of input and output channels for the image. We thus have the kernel $K \in \R^{c_O \times c_I \times k \times k }$ for which we assume that a circular convolution on $h \times w$ will be applied \citep{Gray2005ToeplitzReview, Rao2018TheHandbook}. Assuming circular boundary conditions, the parameterizations of $A, B, \allC$ can be interpreted in terms of the spectral frequency matrices, and in turn, the same normalization operations as above can be applied to them.

\section{Other Architectures}

\subsection{Linear Network}

It is possible to easily derive the same conditions as before for a classical deep neural network of the form,
\begin{gather*}
    x_{k + 1} = w_{k,n}, \\
    v_{k,1 } = C_1 x_{k} + b_1, \\
    w_{k,1} =  \sigma_1(v_{k, 1} ), \\
     \vdots  \\
    v_{k,n } = C_n w_{k, n -1} + b_n, \\
    w_{k,n} =  \sigma_n(v_{k, n} ).
\end{gather*}
By simply setting $A_k = \Zb, B_k = \Ib$ from our residual structure. In turn, this results in the LMI, which generates the following conditions,
\begin{lemma}
The result of $\Db_1$ is equal to the symmetric matrix,
    \begin{align*}
        \Db_1 =&  \Ll^2 \Ib + 2 L_1 m_1 C_{1}^\top  \Lambda_{1} C_{1},
        \end{align*}
     The result of $\Db_j$  for $j \in \{2, \cdots, n\}$ is equal to the symmetric matrix,
    \begin{align*}
        \Db_j =& 2 L_j m_j C_{j}^\top  \Lambda_{j} C_{j} + 2  \Lambda_{j} - (L_{j - 1} + m_{j - 1})^2 \Lambda_{j - 1} C_{j - 1}\Db_{j - 1}^{-1} C_{j - 1}^\top \Lambda_{j - 1},
    \end{align*}
    The block triangular terms of $\Lb_{ij}$  for $j = \{1, \cdots, n\}$, where $\Lb_{jj} = \Ib$, are the following,
    \begin{align*}
    \Lb_{(j + 1)j} &= -(L_j + m_j)   \Lambda_{j} C_{j}\Db_j^{-1},  \\
    \Lb_{(j + 2)j} &= \Zb,  \\
    &\vdots \\
    \Lb_{(n + 1)j} &= \Zb ,
    \end{align*}
    where we define $\Jb_j$ as the following:
The symmetric block diagonal $\Db_n$ is of the form,
    \begin{align*}
        \Db_{n + 1} =& 2\Lambda_n-\Ib -  (L_n + m_n)^2\Lambda_{n}C_{n}\Db_{n}^{-1}C_{n}^\top \Lambda_{n},
    \end{align*}
\end{lemma}
Again assuming $L_i = 1, m_j = 0$ and $\Lambda_j = \Ib$ for all $j$ (same as in Section \ref{sec:ldlt_constraints}), for the sake of simplicity. In turn, we get the following simplified conditions,
\begin{align*}
    \Db_1 &= \Ll^2 \Ib , \\ 
     \Db_j &= 2\Ib -  C_{j - 1}\Db_{j - 1}^{-1} C_{j - 1}^\top , \quad \forall j \in \{2, \cdots, n\}, \\
     \Db_{n + 1} &= \Ib -  C_{n}\Db_{n}^{-1}C_{n}^\top,
\end{align*}
where $D_1$ is clearly always satisfied. In turn, the simplified constraints are defined as,
\begin{align*}
    \kappa_j &= \begin{cases}
        1, & j = n, \\
        2, & \text{else}
    \end{cases}, \\
    C_{j} D_{j}^{-1} C_{j}^\top &\prec  \kappa_j \Ib, \quad \forall j
\end{align*}
which generates parameterization similar to the ones in Section \ref{sec:ldlt_parameterization},
\begin{lemma} \label{lm:feedforward_lipschitz_network}
For the linear network defined by the $\Ll$-Lipschitz activation function $\sigma$, the deep network's weights $\allC$ can be parameterized as,
\begin{align*}
    C_j &= \sqrt{\kappa_j} W_j (\alpha_j \Ib + W_j^\top W_j)^{-\frac{1}{2}} D_{j}^{\frac{1}{2}}, \quad \forall j \\
    C_j &= \sqrt{\kappa_j \kappa_{j-1}} W_j R_j^{-1}  R_{j-1}^{-\top},
\end{align*}
where $R_{j}^\top R_{j} = \bar \Omega_j$.
\begin{align*}
    \bar \Omega_j = \begin{cases}
       \Ib, & \text{if } j = 1 \\
        \alpha_j \Ib + W_j W_j^\top, & \mathrm{else} \\
    \end{cases} .
\end{align*}
\end{lemma}
which also implies that, from the logic following from Lemma \ref{lm:d_j_simpl},
\begin{align*}
    \Db_j &= \kappa_{j -1} (\Ib + W_{j - 1} W_{j - 1}^\top)^{-1}, \quad \forall j \in \{2, \cdots, n\}.
\end{align*}
this more expressive nature to current methodologies such as in \citet{Araujo2023, Prach2022, Meunier2021ANetworks, Singla2021ImprovedCIFAR-100} allow for an end-to-end $\Ll$-parameterization of the linear network, thus helping for mitigating decay initialization problems illustrated in \citet{Juston20251-LipschitzProblem}, where the approximated 1-Lipschitz deep neural network parameterizations of just stacking 1-Lipschitz layers and ensuring the 1-Lipschitz property causes decay due to these stacking approximations. This recursively dependent network, with the addition of the Cholesky decomposition parameterization in Section \ref{sec:sqaure_root_comp}, provides an extremely computationally efficient and robust framework for deep Lipschitz networks.

\subsection{U-Net}

U-Nets \citep{Ronneberger2015U-Net:Segmentation}, a variant of variational latent state encoders (VAE), are effective at encoding information into a lower-dimensional latent space. These networks have been widely used in diffusion networks, data compression, and similar tasks \citep{Rombach2021High-ResolutionModels, Kingma2019AnAutoencoders}. The network structure is as follows, assuming an U-Net of depth $d$, meaning that there are $n = 2d - 1$ blocks, starting with the encoder structure,
%
%
\begin{gather*}
    v_{k,1 } = C_1 x_{k} + b_1, \\
    w_{k,1} =  \sigma_1(v_{k, 1} ), \\
     \vdots  \\
    v_{k,d } = C_d w_{k, d -1} + b_d, \\
    w_{k,d} =  \sigma_d(v_{k, d} ).
\end{gather*}
where $w_{k, d}$ represents the latent space encoding vector output. The decoder output is,
\begin{gather*}
    v_{k,d + 1} = C_{d + 1, 1} w_{k, d -1} + C_{d + 1, 2} w_{k, d}   + b_{d + 1}, \\
    w_{k,d + 1} =  \sigma_1(v_{k, d + 1} ), \\
     \vdots  \\
    v_{k,n } = C_{n, 1}w_{k, 1} + C_{n, 2} w_{k, n -1} + b_n, \\
    w_{k,n} =  \sigma_n(v_{k, n} ), \\
    x_{k + 1} = w_{k, n}
\end{gather*}
from which the full LMI and the set of constraints can be derived; these will be explored in a separate paper. Having a Lipschitz U-Net could provide interesting properties to the latent space representation encoding vector. Given that having a latent space with a Lipschitz bound enforces the distance between embeddings to be a Lipschitz bound as well. This formulation will be explored in future work.

\section{Experiments}

This Section goes through evaluating the $LDL^\top$ network with other state-of-the-art algorithms on a series of data sets from the UCI machine learning repository and an analysis of the results.

\subsection{UCI Machine Learning Repository Dataset}

To validate and compare the performance of the $LDL^\top$ based algorithms, we evaluated this work and other state-of-the-art techniques on a collection of 121 classification data sets from the UCI Machine Learning repository. These repositories encompass a wide variety of problems from physics, geology, and biology \citep{Fernandez-Delgado2014DoProblems}. The data sets range in size from 10 to 63 million data points, with input features ranging from 1 to 3.2 million. Because there are such a large number of data sets and some of the data sets have a prohibitively large amount of data for training, we instead look at a subset of these data sets whose data set subsection selection was provided by the previous works for SeLU \citep{Klambauer2017}, which selects 121 data sets and prepares them with four-fold cross-validation.

\subsection{Details on the Architecture and Hyper-parameters}


To ensure fairness among the chosen algorithms, each network was provided with a set depth of four hidden layers. A simple heuristic determined the network’s width, $w$, where
\begin{align*}
    w &= \min(\max(4N, 32), 512), \\
    w &= (1 + 0.25 \mathbbm{1}_{M > 10}) b,  \\
    w &= 2^{\operatorname{round}(\log_2(b))} ,
\end{align*}
and we have to choose the closest power-of-2 width for these data set widths, Table \ref{tab:model-ranges} demonstrates the different 1-Lipchitz architectures used and the parameter specifications during training across all data sets. In addition, all networks were trained using the AdamW optimizer with a weight decay of 1e-4 and an initial learning rate of 1e-3; a training scheduler was also used that halved the learning rate when validation accuracy plateaued for eight or more epochs. Each network was trained for a maximum of 100 epochs, and early termination was set to terminate the run if no improvement in validation accuracy was observed for 30 epochs. The loss function used was cross-entropy, and class weights were calculated to normalize the data sets' uneven class distributions, especially in the smaller data sets. Each network was four layers deep, as adding more layers did not change rankings and is therefore deemed acceptable.

For simpler, consistent computation of certification accuracy, each network architecture was fitted with a linear normalization head as its last layer.

\begin{table}[htb]
\centering
\begin{threeparttable}
\begin{tabular}{l r r r r r r}
\toprule
Algorithm & Width & Depth & Parameters & Padding & Input dim & Output dim \\
\midrule
AOL & 32---512 & 4---4 & 282---694837 & 10---524 & 3---262 & 2---100 \\
Orthogonal & 32---512 & 4---4 & 285---694840 & 10---524 & 3---262 & 2---100 \\
Sandwich & 32---512 & 4---4 & 615---1520140 & 10---524 & 3---262 & 2---100 \\
SLL & 32---512 & 4---4 & 1558---1084073 & 10---524 & 3---262 & 2---100 \\
LDLT-L & 32---512 & 4---4 & 3366---929297 & 10---524 & 3---262 & 2---100 \\
LDLT-R & 32---512 & 4---4 & 3463---1063442 & 10---524 & 3---262 & 2---100 \\
\bottomrule
\end{tabular}
\end{threeparttable}
\caption{Model dimension ranges (min---max across all data sets and folds). Input/Output dimensions follow data set label spaces.}
\label{tab:model-ranges}
\end{table}

We first compare the mean accuracy of each algorithm to validate purely on accuracy, and we also compare the provable accuracy of each network in Table \ref{tab:metric_summary}. For each model on each data set, we report both clean accuracy and certified robust accuracy at $\ell_2$ radii 36/255, 72/255, 108/255, and 255/255, obtained from the global Lipschitz bound. Given a 1-Lipschitz network and a radius $\epsilon$, we are guaranteed that predictions cannot change within an $\ell_2$ ball of radius $\epsilon$; we therefore label a test point as certifiably correct at radius $\epsilon$ if the prediction of the datapoint matches the true label \citep{Araujo2023}. 
We can see that, in terms of accuracy, the Sandwich layer network consistently performs best among the other algorithms; however, between the LDLT and the SLL, the closest equivalent to this paper, the LDLT-R algorithm, consistently performs better than the SLL-based layers. The AOL algorithm consistently performed the worst among the algorithms. The Tables \ref{tab:overall:mean_test_acc}-\ref{tab:overall:mean_cert_acc_255} provide a statistical comparison of the different algorithms at different certified accuracy noise levels. The average ranks compare relative performance across many data sets without assuming scales \citep{Demsar2006StatisticalSets}. This helps guard against a few easy data sets with large dominating means \citep{Demsar2006StatisticalSets}. The Wilcoxon-based significant loss helps answer how often A performs significantly better than B across data sets, and the mean-aggregate scores show the absolute accuracy margins between the algorithms, which rank-based methods often ignore \citep{Demsar2006StatisticalSets}. 

From the ranking statistics from Tables \ref{tab:overall:mean_test_acc}-\ref{tab:overall:mean_cert_acc_group} we can see that the Sandwich layers clearly dominate all the tests; however, the LDLT-R is always above the SLL layer rankings; however, for the general accuracy metric the LDLT-L barely ties with the Sandwich layers but does not handle the noise as well.

\begin{table}[htb]
\centering
{
  \setlength{\tabcolsep}{2pt}
\begin{threeparttable}
\begin{tabular}{l r lllll}
\toprule
 &  &  & \multicolumn{4}{c}{Certified Accuracy} \\
\cmidrule(lr){4-7}
Algorithm & $N$ & Accuracy & 36/255 & 72/255 & 108/255 & 255/255 \\
\midrule
AOL & 121 & 0.6295\,\tiny$\pm$0.2278 & 0.3669\,\tiny$\pm$0.2895 & 0.2660\,\tiny$\pm$0.2953 & 0.2076\,\tiny$\pm$0.2819 & 0.0999\,\tiny$\pm$0.1875 \\
Orthogonal & 121 & 0.6969\,\tiny$\pm$0.1938 & 0.5973\,\tiny$\pm$0.2386 & 0.5073\,\tiny$\pm$0.2617 & 0.4300\,\tiny$\pm$0.2702 & 0.1970\,\tiny$\pm$0.2288 \\
Sandwich & 121 & 0.7215\,\tiny$\pm$0.1871 & \textbf{0.6375\,\tiny$\pm$0.2305} & \textbf{0.5593\,\tiny$\pm$0.2503} & \textbf{0.4836\,\tiny$\pm$0.2659} & \textbf{0.2496\,\tiny$\pm$0.2471} \\
SLL & 121 & 0.6978\,\tiny$\pm$0.1998 & 0.5885\,\tiny$\pm$0.2451 & 0.4975\,\tiny$\pm$0.2649 & 0.4146\,\tiny$\pm$0.2715 & 0.1918\,\tiny$\pm$0.2222 \\
\midrule
LDLT-L & 121 & \textbf{0.7223\,\tiny$\pm$0.1868} & 0.5301\,\tiny$\pm$0.2920 & 0.4293\,\tiny$\pm$0.3049 & 0.3535\,\tiny$\pm$0.3003 & 0.1652\,\tiny$\pm$0.2281 \\
LDLT-R & 121 & 0.7022\,\tiny$\pm$0.1944 & 0.6107\,\tiny$\pm$0.2314 & 0.5292\,\tiny$\pm$0.2525 & 0.4492\,\tiny$\pm$0.2655 & 0.2172\,\tiny$\pm$0.2312 \\
\bottomrule
\end{tabular}
\end{threeparttable}
}
\caption{Sorted mean$\pm$std across $N$ data sets for each algorithm.}
\label{tab:metric_summary}
\end{table}

\begin{table}[htb]
\centering
\begin{threeparttable}
{\small
\setlength{\tabcolsep}{4pt}
\begin{tabular}{@{}l r r r r r r@{}}
\toprule
Algorithm & \shortstack{Avg \\ rank} $\downarrow$ & \shortstack{sig \\ wins} & \shortstack{sig \\ losses} & \shortstack{net \\ wins} & \shortstack{win \\ share} & mean $r$ \\
\midrule
LDLT-L & 2.434 & 4 & 0 & 4 & 0.800 & 0.577 \\
Sandwich & 2.566 & 4 & 0 & 4 & 0.800 & 0.517 \\
LDLT-R & 3.438 & 1 & 2 & -1 & 0.200 & 0.629 \\
SLL & 3.624 & 1 & 2 & -1 & 0.200 & 0.678 \\
Orthogonal & 3.831 & 1 & 2 & -1 & 0.200 & 0.639 \\
AOL & 5.107 & 0 & 5 & -5 & 0.000 & 0.000 \\
\bottomrule
\end{tabular}
}
\end{threeparttable}
\caption{Overall comparison on Mean Accuracy: average rank (lower is better) with Iman---Davenport $F=44.33$ (df=5,600), $p=1.11e-16$; Nemenyi CD$=0.685$. Counts are significant wins/losses after Holm within-metric at $\alpha=0.05$.}
\label{tab:overall:mean_test_acc}
\end{table}

\begin{table}[htb]
\centering
\begin{subtable}[t]{0.48\linewidth}
\centering
\begin{threeparttable}
{\small
\setlength{\tabcolsep}{4pt}
\begin{adjustbox}{max width=\linewidth}
\begin{tabular}{@{}l r r r r r r@{}}
\toprule
Algorithm & \shortstack{Avg \\ rank} $\downarrow$ & \shortstack{sig \\ wins} & \shortstack{sig \\ losses} & \shortstack{net \\ wins} & \shortstack{win \\ share} & mean $r$ \\
\midrule
Sandwich & 2.021 & 5 & 0 & 5 & 1.000 & 0.626 \\
LDLT-R & 2.715 & 4 & 1 & 3 & 0.800 & 0.514 \\
Orthogonal & 3.417 & 2 & 2 & 0 & 0.400 & 0.627 \\
SLL & 3.426 & 2 & 2 & 0 & 0.400 & 0.602 \\
LDLT-L & 3.785 & 1 & 4 & -3 & 0.200 & 0.732 \\
AOL & 5.636 & 0 & 5 & -5 & 0.000 & 0.000 \\
\bottomrule
\end{tabular}
\end{adjustbox}}
\end{threeparttable}
\caption{Overall comparison on Mean Certified Accuracy (36/255): average rank (lower is better) with Iman---Davenport $F=89.22$ (df=5,600), $p=1.11e-16$; Nemenyi CD$=0.685$. }
\label{tab:overall:mean_cert_acc_36}
\end{subtable}
\hfill
\begin{subtable}[t]{0.48\linewidth}
\centering
\begin{threeparttable}
{\small
\setlength{\tabcolsep}{4pt}
\begin{adjustbox}{max width=\linewidth}
\begin{tabular}{@{}l r r r r r r@{}}
\toprule
Algorithm & \shortstack{Avg \\ rank} $\downarrow$ & \shortstack{sig \\ wins} & \shortstack{sig \\ losses} & \shortstack{net \\ wins} & \shortstack{win \\ share} & mean $r$ \\
\midrule
Sandwich & 1.926 & 5 & 0 & 5 & 1.000 & 0.670 \\
LDLT-R & 2.628 & 4 & 1 & 3 & 0.800 & 0.557 \\
Orthogonal & 3.376 & 2 & 2 & 0 & 0.400 & 0.643 \\
SLL & 3.409 & 2 & 2 & 0 & 0.400 & 0.640 \\
LDLT-L & 4.058 & 1 & 4 & -3 & 0.200 & 0.732 \\
AOL & 5.603 & 0 & 5 & -5 & 0.000 & 0.000 \\
\bottomrule
\end{tabular}
\end{adjustbox}}
\end{threeparttable}
\caption{Overall comparison on Mean Certified Accuracy (72/255): average rank (lower is better) with Iman---Davenport $F=101.00$ (df=5,600), $p=1.11e-16$; Nemenyi CD$=0.685$. }
\label{tab:overall:mean_cert_acc_72}
\end{subtable}
\hfill
\begin{subtable}[t]{0.48\linewidth}
\centering
\begin{threeparttable}
{\small
\setlength{\tabcolsep}{4pt}
\begin{adjustbox}{max width=\linewidth}
\begin{tabular}{@{}l r r r r r r@{}}
\toprule
Algorithm & \shortstack{Avg \\ rank} $\downarrow$ & \shortstack{sig \\ wins} & \shortstack{sig \\ losses} & \shortstack{net \\ wins} & \shortstack{win \\ share} & mean $r$ \\
\midrule
Sandwich & 1.868 & 5 & 0 & 5 & 1.000 & 0.675 \\
LDLT-R & 2.562 & 4 & 1 & 3 & 0.800 & 0.548 \\
Orthogonal & 3.347 & 2 & 2 & 0 & 0.400 & 0.674 \\
SLL & 3.360 & 2 & 2 & 0 & 0.400 & 0.660 \\
LDLT-L & 4.136 & 1 & 4 & -3 & 0.200 & 0.801 \\
AOL & 5.727 & 0 & 5 & -5 & 0.000 & 0.000 \\
\bottomrule
\end{tabular}
\end{adjustbox}}
\end{threeparttable}
\caption{Overall comparison on Mean Certified Accuracy (108/255): average rank (lower is better) with Iman---Davenport $F=125.70$ (df=5,600), $p=1.11e-16$; Nemenyi CD$=0.685$. }
\label{tab:overall:mean_cert_acc_108}
\end{subtable}
\hfill
\begin{subtable}[t]{0.48\linewidth}
\centering
\begin{threeparttable}
{\small
\setlength{\tabcolsep}{4pt}
\begin{adjustbox}{max width=\linewidth}
\begin{tabular}{@{}l r r r r r r@{}}
\toprule
Algorithm & \shortstack{Avg \\ rank} $\downarrow$ & \shortstack{sig \\ wins} & \shortstack{sig \\ losses} & \shortstack{net \\ wins} & \shortstack{win \\ share} & mean $r$ \\
\midrule
Sandwich & 1.909 & 5 & 0 & 5 & 1.000 & 0.671 \\
LDLT-R & 2.550 & 4 & 1 & 3 & 0.800 & 0.575 \\
SLL & 3.310 & 2 & 2 & 0 & 0.400 & 0.662 \\
Orthogonal & 3.471 & 2 & 2 & 0 & 0.400 & 0.607 \\
LDLT-L & 4.202 & 1 & 4 & -3 & 0.200 & 0.796 \\
AOL & 5.558 & 0 & 5 & -5 & 0.000 & 0.000 \\
\bottomrule
\end{tabular}
\end{adjustbox}}
\end{threeparttable}
\caption{Overall comparison on Mean Certified Accuracy (255/255): average rank (lower is better) with Iman---Davenport $F=105.79$ (df=5,600), $p=1.11e-16$; Nemenyi CD$=0.685$. }
\label{tab:overall:mean_cert_acc_255}
\end{subtable}
\caption{Overall comparison on Mean Certified Accuracy at perturbation radii 36/255, 72/255, 108/255, and 255/255. Counts are significant wins/losses after Holm within-metric at $\alpha=0.05$.}
\label{tab:overall:mean_cert_acc_group}
\end{table}


The Tables \ref{tab:signif:mean_test_acc}-\ref{tab:signif:mean_cert_acc_group} show a graphical comparison of the pair-wise results between each of the algorithms from the Wilcoxon pair-wise test. The full summary of the results is provided in Appendix \ref{app:wilcoxon_full}. We can again see that, for the non-perturbed mean accuracy metric (Table \ref{tab:signif:mean_test_acc}), the LDLT-L. LDLT-R and Sandwich layers performed similarly, with the same number of victories. However, when looking at the perturbed certification accuracies, the Sandwich layers perform better than every other technique, while the LDLT-R only loses to the Sandwich layers.


\begin{table}[htb]
\centering
\begin{threeparttable}
{
\setlength{\tabcolsep}{3pt}
\begin{tabular}{@{}l c c c c c c @{}}
\toprule
 & AOL & LDLT-L & LDLT-R & Orthogonal & Sandwich & SLL \\
\midrule
AOL & $\cdot$ & \textcolor{red}{$\blacktriangledown$} & \textcolor{red}{$\blacktriangledown$} & \textcolor{red}{$\blacktriangledown$} & \textcolor{red}{$\blacktriangledown$} & \textcolor{red}{$\blacktriangledown$} \\
LDLT-L & \textcolor{green}{$\blacktriangle$} & $\cdot$ & \textcolor{green}{$\blacktriangle$} & \textcolor{green}{$\blacktriangle$} & $\cdot$ & \textcolor{green}{$\blacktriangle$} \\
LDLT-R & \textcolor{green}{$\blacktriangle$} & \textcolor{red}{$\blacktriangledown$} & $\cdot$ & $\cdot$ & \textcolor{red}{$\blacktriangledown$} & $\cdot$ \\
Orthogonal & \textcolor{green}{$\blacktriangle$} & \textcolor{red}{$\blacktriangledown$} & $\cdot$ & $\cdot$ & \textcolor{red}{$\blacktriangledown$} & $\cdot$ \\
Sandwich & \textcolor{green}{$\blacktriangle$} & $\cdot$ & \textcolor{green}{$\blacktriangle$} & \textcolor{green}{$\blacktriangle$} & $\cdot$ & \textcolor{green}{$\blacktriangle$} \\
SLL & \textcolor{green}{$\blacktriangle$} & \textcolor{red}{$\blacktriangledown$} & $\cdot$ & $\cdot$ & \textcolor{red}{$\blacktriangledown$} & $\cdot$ \\
\bottomrule
\end{tabular}
}
\end{threeparttable}
\caption{Pairwise Wilcoxon outcomes for Mean Accuracy (Holm within-metric at $\alpha=0.05$): row vs. column (\textcolor{green}{$\blacktriangle$} win, \textcolor{red}{$\blacktriangledown$} loss, $\cdot$ none).}
\label{tab:signif:mean_test_acc}
\end{table}

\begin{table}[htb]
\centering
\begin{subtable}[t]{0.48\linewidth}
\centering
\begin{threeparttable}
{
\setlength{\tabcolsep}{3pt}
\begin{adjustbox}{max width=\linewidth}
\begin{tabular}{@{}l c c c c c c @{}}
\toprule
 & AOL & LDLT-L & LDLT-R & Orthogonal & Sandwich & SLL \\
\midrule
AOL & $\cdot$ & \textcolor{red}{$\blacktriangledown$} & \textcolor{red}{$\blacktriangledown$} & \textcolor{red}{$\blacktriangledown$} & \textcolor{red}{$\blacktriangledown$} & \textcolor{red}{$\blacktriangledown$} \\
LDLT-L & \textcolor{green}{$\blacktriangle$} & $\cdot$ & \textcolor{red}{$\blacktriangledown$} & \textcolor{red}{$\blacktriangledown$} & \textcolor{red}{$\blacktriangledown$} & \textcolor{red}{$\blacktriangledown$} \\
LDLT-R & \textcolor{green}{$\blacktriangle$} & \textcolor{green}{$\blacktriangle$} & $\cdot$ & \textcolor{green}{$\blacktriangle$} & \textcolor{red}{$\blacktriangledown$} & \textcolor{green}{$\blacktriangle$} \\
Orthogonal & \textcolor{green}{$\blacktriangle$} & \textcolor{green}{$\blacktriangle$} & \textcolor{red}{$\blacktriangledown$} & $\cdot$ & \textcolor{red}{$\blacktriangledown$} & $\cdot$ \\
Sandwich & \textcolor{green}{$\blacktriangle$} & \textcolor{green}{$\blacktriangle$} & \textcolor{green}{$\blacktriangle$} & \textcolor{green}{$\blacktriangle$} & $\cdot$ & \textcolor{green}{$\blacktriangle$} \\
SLL & \textcolor{green}{$\blacktriangle$} & \textcolor{green}{$\blacktriangle$} & \textcolor{red}{$\blacktriangledown$} & $\cdot$ & \textcolor{red}{$\blacktriangledown$} & $\cdot$ \\
\bottomrule
\end{tabular}
\end{adjustbox}}
\end{threeparttable}
\caption{Pairwise Wilcoxon outcomes for Mean Certified Accuracy (36/255).}
\label{tab:signif:mean_cert_acc_36}
\end{subtable}
\hfill
\begin{subtable}[t]{0.48\linewidth}
\centering
\begin{threeparttable}
{
\setlength{\tabcolsep}{3pt}
\begin{adjustbox}{max width=\linewidth}
\begin{tabular}{@{}l c c c c c c @{}}
\toprule
 & AOL & LDLT-L & LDLT-R & Orthogonal & Sandwich & SLL \\
\midrule
AOL & $\cdot$ & \textcolor{red}{$\blacktriangledown$} & \textcolor{red}{$\blacktriangledown$} & \textcolor{red}{$\blacktriangledown$} & \textcolor{red}{$\blacktriangledown$} & \textcolor{red}{$\blacktriangledown$} \\
LDLT-L & \textcolor{green}{$\blacktriangle$} & $\cdot$ & \textcolor{red}{$\blacktriangledown$} & \textcolor{red}{$\blacktriangledown$} & \textcolor{red}{$\blacktriangledown$} & \textcolor{red}{$\blacktriangledown$} \\
LDLT-R & \textcolor{green}{$\blacktriangle$} & \textcolor{green}{$\blacktriangle$} & $\cdot$ & \textcolor{green}{$\blacktriangle$} & \textcolor{red}{$\blacktriangledown$} & \textcolor{green}{$\blacktriangle$} \\
Orthogonal & \textcolor{green}{$\blacktriangle$} & \textcolor{green}{$\blacktriangle$} & \textcolor{red}{$\blacktriangledown$} & $\cdot$ & \textcolor{red}{$\blacktriangledown$} & $\cdot$ \\
Sandwich & \textcolor{green}{$\blacktriangle$} & \textcolor{green}{$\blacktriangle$} & \textcolor{green}{$\blacktriangle$} & \textcolor{green}{$\blacktriangle$} & $\cdot$ & \textcolor{green}{$\blacktriangle$} \\
SLL & \textcolor{green}{$\blacktriangle$} & \textcolor{green}{$\blacktriangle$} & \textcolor{red}{$\blacktriangledown$} & $\cdot$ & \textcolor{red}{$\blacktriangledown$} & $\cdot$ \\
\bottomrule
\end{tabular}
\end{adjustbox}}
\end{threeparttable}
\caption{Pairwise Wilcoxon outcomes for Mean Certified Accuracy (72/255).}
\label{tab:signif:mean_cert_acc_72}
\end{subtable}
\hfill
\begin{subtable}[t]{0.48\linewidth}
\centering
\begin{threeparttable}
{
\setlength{\tabcolsep}{3pt}
\begin{adjustbox}{max width=\linewidth}
\begin{tabular}{@{}l c c c c c c @{}}
\toprule
 & AOL & LDLT-L & LDLT-R & Orthogonal & Sandwich & SLL \\
\midrule
AOL & $\cdot$ & \textcolor{red}{$\blacktriangledown$} & \textcolor{red}{$\blacktriangledown$} & \textcolor{red}{$\blacktriangledown$} & \textcolor{red}{$\blacktriangledown$} & \textcolor{red}{$\blacktriangledown$} \\
LDLT-L & \textcolor{green}{$\blacktriangle$} & $\cdot$ & \textcolor{red}{$\blacktriangledown$} & \textcolor{red}{$\blacktriangledown$} & \textcolor{red}{$\blacktriangledown$} & \textcolor{red}{$\blacktriangledown$} \\
LDLT-R & \textcolor{green}{$\blacktriangle$} & \textcolor{green}{$\blacktriangle$} & $\cdot$ & \textcolor{green}{$\blacktriangle$} & \textcolor{red}{$\blacktriangledown$} & \textcolor{green}{$\blacktriangle$} \\
Orthogonal & \textcolor{green}{$\blacktriangle$} & \textcolor{green}{$\blacktriangle$} & \textcolor{red}{$\blacktriangledown$} & $\cdot$ & \textcolor{red}{$\blacktriangledown$} & $\cdot$ \\
Sandwich & \textcolor{green}{$\blacktriangle$} & \textcolor{green}{$\blacktriangle$} & \textcolor{green}{$\blacktriangle$} & \textcolor{green}{$\blacktriangle$} & $\cdot$ & \textcolor{green}{$\blacktriangle$} \\
SLL & \textcolor{green}{$\blacktriangle$} & \textcolor{green}{$\blacktriangle$} & \textcolor{red}{$\blacktriangledown$} & $\cdot$ & \textcolor{red}{$\blacktriangledown$} & $\cdot$ \\
\bottomrule
\end{tabular}
\end{adjustbox}}
\end{threeparttable}
\caption{Pairwise Wilcoxon outcomes for Mean Certified Accuracy (108/255).}
\label{tab:signif:mean_cert_acc_108}
\end{subtable}
\hfill
\begin{subtable}[t]{0.48\linewidth}
\centering
\begin{threeparttable}
{
\setlength{\tabcolsep}{3pt}
\begin{adjustbox}{max width=\linewidth}
\begin{tabular}{@{}l c c c c c c @{}}
\toprule
 & AOL & LDLT-L & LDLT-R & Orthogonal & Sandwich & SLL \\
\midrule
AOL & $\cdot$ & \textcolor{red}{$\blacktriangledown$} & \textcolor{red}{$\blacktriangledown$} & \textcolor{red}{$\blacktriangledown$} & \textcolor{red}{$\blacktriangledown$} & \textcolor{red}{$\blacktriangledown$} \\
LDLT-L & \textcolor{green}{$\blacktriangle$} & $\cdot$ & \textcolor{red}{$\blacktriangledown$} & \textcolor{red}{$\blacktriangledown$} & \textcolor{red}{$\blacktriangledown$} & \textcolor{red}{$\blacktriangledown$} \\
LDLT-R & \textcolor{green}{$\blacktriangle$} & \textcolor{green}{$\blacktriangle$} & $\cdot$ & \textcolor{green}{$\blacktriangle$} & \textcolor{red}{$\blacktriangledown$} & \textcolor{green}{$\blacktriangle$} \\
Orthogonal & \textcolor{green}{$\blacktriangle$} & \textcolor{green}{$\blacktriangle$} & \textcolor{red}{$\blacktriangledown$} & $\cdot$ & \textcolor{red}{$\blacktriangledown$} & $\cdot$ \\
Sandwich & \textcolor{green}{$\blacktriangle$} & \textcolor{green}{$\blacktriangle$} & \textcolor{green}{$\blacktriangle$} & \textcolor{green}{$\blacktriangle$} & $\cdot$ & \textcolor{green}{$\blacktriangle$} \\
SLL & \textcolor{green}{$\blacktriangle$} & \textcolor{green}{$\blacktriangle$} & \textcolor{red}{$\blacktriangledown$} & $\cdot$ & \textcolor{red}{$\blacktriangledown$} & $\cdot$ \\
\bottomrule
\end{tabular}
\end{adjustbox}}
\end{threeparttable}
\caption{Pairwise Wilcoxon outcomes for Mean Certified Accuracy (255/255).}
\label{tab:signif:mean_cert_acc_255}
\end{subtable}
\caption{Overall pairwise Wilcoxon outcomes on Mean Certified Accuracy at perturbation radii 36/255, 72/255, 108/255, and 255/255. row vs. column (\textcolor{green}{$\blacktriangle$} win, \textcolor{red}{$\blacktriangledown$} loss, $\cdot$ tie). Holm within-metric at $\alpha=0.05$.}
\label{tab:signif:mean_cert_acc_group}
\end{table}

\section{Limitations}

ResNet is numerically unstable in deep networks unless more robust Cholesky-update routines are implemented. The convolution implementation of the network is computationally expensive due to the cost of complex operators and the doubling of matrix sizes in the block-matrix representation. PyTorch is currently working on implementing less computationally expensive complex operations (cf. https://github.com/openteams-ai/pytorch-complex-tensor, for details).
%
Training networks for complex data sets like CIFAR10, CIFAR100, or TinyImageNet becomes prohibitively expensive due to the scaling of CNN computational cost. These data sets require greater convolutional channel depths for effectiveness. Future work involves developing these additional data sets and validations.

\section{Conclusion}

We introduced a novel $LDL^\top$ framework to decompose the LMI into block matrices D, the only components needing constraints. This enables complex architectures to derive and enforce their constraints. The deep residual network structure and deep linear network are demonstrated through an end-to-end analysis, not a shallow layer-by-layer perspective.

This architecture offers greater robustness to adversarial noise than its direct competitor, the SDP-Layers \citep{Araujo2023}, and outperforms other state-of-the-art algorithms. The residual structure LDLT-R consistently ranks higher than the SLL layers, achieving 3\%-13\% better performance at higher certified accuracy thresholds. The Sandwich layers (\citealp{Wang2023DirectNetworks}) perform best among the selected algorithms. Based on the empirical accuracy increase from the Sandwich layers, we will explore integrating their parameterization into our system to achieve further accuracy. In the future, U-Nets, with their 1-Lipschitz latent-state representation, could be considered for specific applications or unknown beneficial properties. Novel normalization schemes based on the current network architecture will be derived to improve training efficiency. From the $LDL^\top$ decomposition formulation, this methodology enables decomposing neural networks into linear and non-linear operations and generating both the LMI and its parametric $LDL^\top$ decomposition for general architectures.

This methodology, derived from the $LDL^\top$ decomposition formulation, enables the creation of an algorithmic method for decomposing neural networks into linear and non-linear operations, generating both the LMI and its parametric $LDL^\top$ decomposition for general architectures.
%
%
%
%


\section{Code}

The code for this repository, including the models and their respective training weights, is located at \url{https://github.com/Marius-Juston/DeepLipschitzResNet}.
With the trained weights and Tensorboard plots available for download at the HuggingFace repository location \url{https://huggingface.co/SuperComputer/LDLT}.



\section{Declaration of Generative AI and AI-assisted technologies in the writing process}

During the preparation of this work, the authors used ChatGPT (OpenAI) to improve the clarity and fluency of the English text. After using this tool, the author reviewed and edited the content as needed and takes full responsibility for the publication's content.

\acks{This research was supported in part by the U.S. Army Corps of Engineers Engineering Research and Development Center, Construction Engineering Research Laboratory under Grant W9132T23C0013. \\
The authors declare that they have no known competing financial interests or personal relationships that could have appeared to influence
the work reported in this paper.}



\appendix

\section{Activation Function Quadratic Bounds} \label{sec:activation_quad_bounds}

For the sake of completeness, the $L$ and $m$ constants of the activation functions defined in \href{https://pytorch.org/docs/stable/nn.html#non-linear-activations-weighted-sum-nonlinearity}{PyTorch} (assuming default values if not specified) were derived and defined in Table \ref{tab:activation_function_convecities}. It should be noted that the Hardshrink and RReLU could not be used due to their infinite $L, m$ constants; Hardshrink has infinite $L, m$ due to its noncontinuous piece-wise definition, and PReLU due to its stochastic definition, which no longer made its $L, m$ computable. To compute the bounds $L, m$ bounds of the activation functions revolved to finding the minimum and maximum gradient of each activation function, $f(x)$.
\begin{align}
    m = \arg \min_x \frac{\partial f(x)}{\partial x}, \quad L = \arg \max_x \frac{\partial f(x)}{\partial x}
\end{align}
Where,
\begin{align}
    \text{erfc}(z) &= 1 - \text{erf}(z), \quad \text{erf}(z) = \frac{2}{\sqrt{\pi}} \int_{0}^z e^{-t^2}dt.
\end{align}
\begin{table}[htb]
\centering
\renewcommand{\arraystretch}{1.1}
\resizebox{\columnwidth}{!}{%
\begin{tabular}{lcccc}
\toprule
\textbf{Activation Function} & $\mathbf{L}$ & $\mathbf{m}$ & $\mathbf{S}$ & $\mathbf{P}$ \\
\midrule
ELU ($\alpha=1$) \citep{Clevert2015} & $\max(1,\alpha)$ & $0$ & $\max(1,\alpha)$ & $0$ \\
Hardshrink \citep{Cancino2002} & $\infty$ & $0$ & $\infty$ & $\infty$ \\
Hardsigmoid \citep{Courbariaux2015} & $\tfrac{1}{6}$ & $0$ & $\tfrac{1}{6}$ & $0$ \\
Hardtanh \citep{collobert2004} & $1$ & $0$ & $1$ & $0$ \\
Hardswish \citep{Howard2019} & $1.5$ & $-0.5$ & $1$ & $-0.75$ \\
LeakyReLU ($\alpha=10^{-2}$) \citep{Maas2013} & $1$ & $\alpha$ & $1+\alpha$ & $\alpha$ \\
LogSigmoid & $1$ & $0$ & $1$ & $0$ \\
PReLU ($\alpha=\tfrac{1}{4}$) \citep{He2015PReLU} & $1$ & $\alpha$ & $1+\alpha$ & $\alpha$ \\
ReLU \citep{McCulloch1943} & $1$ & $0$ & $1$ & $0$ \\
ReLU6 \citep{Howard2017} & $1$ & $0$ & $1$ & $0$ \\
RReLU \citep{Xu2015} & $\infty$ & $-\infty$ & $\infty$ & $\infty$ \\
SELU \citep{Klambauer2017} & \shortstack{$\alpha\cdot\text{scale}$\\$\approx1.758099341$} & $0$ & \shortstack{$\alpha\cdot\text{scale}$\\$\approx1.758099341$} & $0$ \\
CELU \citep{Barron2017} & $1$ & $0$ & $1$ & $0$ \\
GELU \citep{Hendrycks2016} &
\shortstack{$\dfrac{\mathrm{erfc}(1)}{2}-\dfrac{1}{e\sqrt{\pi}}$\\$\approx 1.128904145$} &
\shortstack{$\dfrac{1}{2}\big(\mathrm{erf}(1)+1\big)+\dfrac{1}{e\sqrt{\pi}}$\\$\approx -0.1289041452$} &
$1$ &
\shortstack{$\dfrac{\big(e\sqrt{\pi}\,(\mathrm{erf}(1)+1)+2\big)\big(e\sqrt{\pi}\,\mathrm{erfc}(1)-2\big)}{4e^{2}\pi}$\\$\approx -0.145520424$} \\
Sigmoid \citep{Sak2014} & $1$ & $0$ & $1$ & $0$ \\
SiLU \citep{Elfwing2017} & $1.099839320$ & $-0.09983932013$ & $1$ & $-0.1098072100$ \\
Softplus \citep{Zhou2016} & $1$ & $0$ & $1$ & $0$ \\
Mish ($\alpha\ge\tfrac{1}{2}$) \citep{Misra2019} & $1.199678640$ & $-0.2157287822$ & $0.8060623125$ & $-0.2204297485$ \\
Softshrink \citep{Cancino2002} & $1$ & $0$ & $1$ & $0$ \\
Softsign \citep{Ping2017} & $1$ & $0$ & $1$ & $0$ \\
Tanh \citep{Sak2014} & $1$ & $0$ & $1$ & $0$ \\
Tanhshrink & $1$ & $0$ & $1$ & $0$ \\
Threshold & $1$ & $0$ & $1$ & $0$ \\
\bottomrule
\end{tabular}
}%
\caption{Convexity constants of the element-wise activation functions in PyTorch.}
\label{tab:activation_function_convecities}
\end{table}

\section{Wilcoxon Metrics} \label{app:wilcoxon_full}

The general metrics for all runs, with more details on pairwise Wilcoxon comparisons for all 121 UCI data sets between each algorithm can be found in Tables \ref{tab:wilcoxon:mean_test_acc}-\ref{tab:wilcoxon:mean_cert_acc_255}, which shows the complete decompositions of win and loss ratios between the different algorithms over all 121 UCI data sets.


\begin{table}[htb]
\centering
\begin{threeparttable}
\begingroup
\setlength{\tabcolsep}{3pt}
\footnotesize
\begin{adjustbox}{max width=\linewidth}
\begin{tabular}{@{}llrrrrrrrrrrr@{}}
\toprule
\multicolumn{2}{c}{Algorithms} & \multicolumn{6}{c}{Run Statistics} & \multicolumn{5}{c}{Wilcoxon pairwise Statistics} \\\cmidrule(lr){1-2} \cmidrule(lr){3-8} \cmidrule(lr){9-13}
Alg A & Alg B & $n$ & wins$_A$ & wins$_B$ & ties & WinRate A & \shortstack{Median \\ $\Delta$ (A---B)} & $W$ & $p$ & $p_{\text{Holm,within}}$ & $p_{\text{Holm,global}}$ & $r$ \\
\midrule
AOL & LDLT-L & 121 & 16 & 105 & 0 & 0.1322 & -0.0568 & 412 & $2.3e-17^{***}$ & $3.4e-16^{***}$ & $0^{***}$ & 0.7708 \\
AOL & Sandwich & 121 & 18 & 103 & 0 & 0.1488 & -0.0595 & 518 & $2.3e-16^{***}$ & $0^{***}$ & $0^{***}$ & 0.7459 \\
AOL & SLL & 121 & 22 & 99 & 0 & 0.1818 & -0.0351 & 806 & $0^{***}$ & $0^{***}$ & $0^{***}$ & 0.6782 \\
AOL & Orthogonal & 121 & 27 & 94 & 0 & 0.2231 & -0.0265 & 971 & $0^{***}$ & $0^{***}$ & $1.0e-10^{***}$ & 0.6394 \\
AOL & LDLT-R & 121 & 25 & 96 & 0 & 0.2066 & -0.0313 & 1013 & $0^{***}$ & $0^{***}$ & $2.0e-10^{***}$ & 0.6295 \\
Orthogonal & Sandwich & 121 & 31 & 89 & 1 & 0.2603 & -0.0126 & 1269 & $6.0e-10^{***}$ & $6.3e-09^{***}$ & $2.2e-08^{***}$ & 0.5643 \\
LDLT-L & Orthogonal & 121 & 94 & 27 & 0 & 0.7769 & 0.0139 & 1309 & $7.0e-10^{***}$ & $6.6e-09^{***}$ & $2.5e-08^{***}$ & 0.5599 \\
LDLT-L & SLL & 121 & 86 & 35 & 0 & 0.7107 & 0.0111 & 1503 & $1.5e-08^{***}$ & $1.2e-07^{***}$ & $4.9e-07^{***}$ & 0.5143 \\
LDLT-L & LDLT-R & 121 & 87 & 33 & 1 & 0.7231 & 0.0108 & 1689 & $3.7e-07^{***}$ & $2.6e-06^{***}$ & $1.1e-05^{***}$ & 0.4639 \\
Sandwich & SLL & 121 & 83 & 38 & 0 & 0.6860 & 0.0090 & 1789 & $8.8e-07^{***}$ & $5.3e-06^{***}$ & $2.3e-05^{***}$ & 0.4470 \\
LDLT-R & Sandwich & 121 & 43 & 78 & 0 & 0.3554 & -0.0068 & 2364 & $6.0e-04^{***}$ & $3.0e-03^{**}$ & $7.2e-03^{**}$ & 0.3118 \\
LDLT-R & SLL & 121 & 66 & 54 & 1 & 0.5496 & 0.0019 & 3070 & $1.4e-01$ & $5.7e-01$ & $8.7e-01$ & 0.1338 \\
LDLT-R & Orthogonal & 121 & 71 & 50 & 0 & 0.5868 & 0.0025 & 3138 & $1.5e-01$ & $5.7e-01$ & $8.7e-01$ & 0.1298 \\
Orthogonal & SLL & 121 & 60 & 61 & 0 & 0.4959 & -0.0004 & 3623 & $8.6e-01$ & $1.0e+00$ & $1.0e+00$ & 0.0158 \\
LDLT-L & Sandwich & 121 & 59 & 62 & 0 & 0.4876 & -0.0002 & 3624 & $8.6e-01$ & $1.0e+00$ & $1.0e+00$ & 0.0155 \\
\bottomrule
\end{tabular}
\end{adjustbox}
\endgroup
\begin{tablenotes}\item Stars mark significance ($^*\,p\!\le\!0.05$, $^{**}\,p\!\le\!0.01$, $^{***}\,p\!\le\!0.001$).
\end{tablenotes}
\end{threeparttable}
\caption[Mean Accuracy]{Wilcoxon signed-rank tests (two-sided) for Mean Accuracy; $p$-values with Holm FWER corrections within-metric and global.}
\label{tab:wilcoxon:mean_test_acc}
\end{table}

\begin{table}[htb]
\centering
\begin{threeparttable}
\begingroup
\setlength{\tabcolsep}{3pt}
\footnotesize
\begin{adjustbox}{max width=\linewidth}
\begin{tabular}{@{}llrrrrrrrrrrr@{}}
\toprule
\multicolumn{2}{c}{Algorithms} & \multicolumn{6}{c}{Run Statistics} & \multicolumn{5}{c}{Wilcoxon pairwise Statistics} \\\cmidrule(lr){1-2} \cmidrule(lr){3-8} \cmidrule(lr){9-13}
Alg A & Alg B & $n$ & wins$_A$ & wins$_B$ & ties & WinRate A & \shortstack{Median \\ $\Delta$ (A---B)} & $W$ & $p$ & $p_{\text{Holm,within}}$ & $p_{\text{Holm,global}}$ & $r$ \\
\midrule
AOL & LDLT-R & 121 & 6 & 115 & 0 & 0.0496 & -0.2319 & 84 & $1.1e-20^{***}$ & $1.6e-19^{***}$ & $7.7e-19^{***}$ & 0.8479 \\
AOL & Sandwich & 121 & 6 & 115 & 0 & 0.0496 & -0.2436 & 88 & $1.2e-20^{***}$ & $1.7e-19^{***}$ & $8.4e-19^{***}$ & 0.8470 \\
AOL & SLL & 121 & 6 & 115 & 0 & 0.0496 & -0.1912 & 145 & $4.8e-20^{***}$ & $6.2e-19^{***}$ & $3.2e-18^{***}$ & 0.8336 \\
AOL & Orthogonal & 121 & 10 & 111 & 0 & 0.0826 & -0.2187 & 196 & $1.6e-19^{***}$ & $1.9e-18^{***}$ & $1.0e-17^{***}$ & 0.8216 \\
AOL & LDLT-L & 121 & 16 & 105 & 0 & 0.1322 & -0.1218 & 579 & $8.5e-16^{***}$ & $0^{***}$ & $0^{***}$ & 0.7315 \\
LDLT-L & Sandwich & 121 & 28 & 93 & 0 & 0.2314 & -0.0511 & 884 & $0^{***}$ & $0^{***}$ & $0^{***}$ & 0.6598 \\
Orthogonal & Sandwich & 121 & 24 & 96 & 1 & 0.2025 & -0.0275 & 1002 & $0^{***}$ & $0^{***}$ & $2.0e-10^{***}$ & 0.6282 \\
Sandwich & SLL & 121 & 97 & 24 & 0 & 0.8017 & 0.0199 & 1005 & $0^{***}$ & $0^{***}$ & $2.0e-10^{***}$ & 0.6314 \\
LDLT-L & LDLT-R & 121 & 39 & 82 & 0 & 0.3223 & -0.0372 & 1356 & $1.6e-09^{***}$ & $1.1e-08^{***}$ & $5.2e-08^{***}$ & 0.5488 \\
LDLT-L & Orthogonal & 121 & 47 & 74 & 0 & 0.3884 & -0.0190 & 1855 & $2.1e-06^{***}$ & $1.2e-05^{***}$ & $4.8e-05^{***}$ & 0.4315 \\
LDLT-R & SLL & 121 & 81 & 39 & 1 & 0.6736 & 0.0118 & 2104 & $6.5e-05^{***}$ & $2.5e-04^{***}$ & $1.0e-03^{**}$ & 0.3647 \\
LDLT-L & SLL & 121 & 49 & 72 & 0 & 0.4050 & -0.0075 & 2119 & $4.8e-05^{***}$ & $2.4e-04^{***}$ & $9.2e-04^{***}$ & 0.3694 \\
LDLT-R & Sandwich & 121 & 41 & 80 & 0 & 0.3388 & -0.0168 & 2142 & $6.2e-05^{***}$ & $2.5e-04^{***}$ & $1.0e-03^{**}$ & 0.3640 \\
LDLT-R & Orthogonal & 121 & 78 & 43 & 0 & 0.6446 & 0.0126 & 2429 & $1.1e-03^{**}$ & $2.2e-03^{**}$ & $1.1e-02^{*}$ & 0.2965 \\
Orthogonal & SLL & 121 & 60 & 61 & 0 & 0.4959 & -0.0004 & 3076 & $1.1e-01$ & $1.1e-01$ & $8.7e-01$ & 0.1444 \\
\bottomrule
\end{tabular}
\end{adjustbox}
\endgroup
\begin{tablenotes}\item Stars mark significance ($^*\,p\!\le\!0.05$, $^{**}\,p\!\le\!0.01$, $^{***}\,p\!\le\!0.001$).
\end{tablenotes}
\end{threeparttable}
\caption[Mean Certified Accuracy (36/255)]{Wilcoxon signed-rank tests (two-sided) for Mean Certified Accuracy (36/255); $p$-values with Holm FWER corrections within-metric and global.}
\label{tab:wilcoxon:mean_cert_acc_36}
\end{table}

\begin{table}[htb]
\centering
\begin{threeparttable}
\begingroup
\setlength{\tabcolsep}{3pt}
\footnotesize
\begin{adjustbox}{max width=\linewidth}
\begin{tabular}{@{}llrrrrrrrrrrr@{}}
\toprule
\multicolumn{2}{c}{Algorithms} & \multicolumn{6}{c}{Run Statistics} & \multicolumn{5}{c}{Wilcoxon pairwise Statistics} \\\cmidrule(lr){1-2} \cmidrule(lr){3-8} \cmidrule(lr){9-13}
Alg A & Alg B & $n$ & wins$_A$ & wins$_B$ & ties & WinRate A & \shortstack{Median \\ $\Delta$ (A---B)} & $W$ & $p$ & $p_{\text{Holm,within}}$ & $p_{\text{Holm,global}}$ & $r$ \\
\midrule
AOL & Sandwich & 121 & 7 & 114 & 0 & 0.0579 & -0.2614 & 107 & $1.9e-20^{***}$ & $2.9e-19^{***}$ & $1.3e-18^{***}$ & 0.8425 \\
AOL & LDLT-R & 121 & 5 & 116 & 0 & 0.0413 & -0.2327 & 116 & $2.4e-20^{***}$ & $3.3e-19^{***}$ & $1.6e-18^{***}$ & 0.8404 \\
AOL & SLL & 121 & 5 & 116 & 0 & 0.0413 & -0.1928 & 162 & $7.1e-20^{***}$ & $9.3e-19^{***}$ & $4.6e-18^{***}$ & 0.8296 \\
AOL & Orthogonal & 121 & 11 & 110 & 0 & 0.0909 & -0.2064 & 214 & $2.4e-19^{***}$ & $2.9e-18^{***}$ & $1.5e-17^{***}$ & 0.8174 \\
AOL & LDLT-L & 121 & 19 & 100 & 2 & 0.1653 & -0.1250 & 557 & $0^{***}$ & $0^{***}$ & $0^{***}$ & 0.7323 \\
Orthogonal & Sandwich & 121 & 19 & 101 & 1 & 0.1612 & -0.0346 & 601 & $0^{***}$ & $0^{***}$ & $0^{***}$ & 0.7240 \\
LDLT-L & Sandwich & 121 & 20 & 101 & 0 & 0.1653 & -0.0772 & 607 & $0^{***}$ & $0^{***}$ & $0^{***}$ & 0.7250 \\
Sandwich & SLL & 121 & 95 & 25 & 1 & 0.7893 & 0.0393 & 801 & $0^{***}$ & $0^{***}$ & $0^{***}$ & 0.6762 \\
LDLT-L & LDLT-R & 121 & 30 & 91 & 0 & 0.2479 & -0.0627 & 930 & $0^{***}$ & $0^{***}$ & $0^{***}$ & 0.6490 \\
LDLT-L & Orthogonal & 121 & 42 & 79 & 0 & 0.3471 & -0.0210 & 1696 & $2.5e-07^{***}$ & $1.5e-06^{***}$ & $7.5e-06^{***}$ & 0.4689 \\
LDLT-L & SLL & 121 & 42 & 79 & 0 & 0.3471 & -0.0213 & 1776 & $7.4e-07^{***}$ & $3.7e-06^{***}$ & $2.0e-05^{***}$ & 0.4501 \\
LDLT-R & SLL & 121 & 83 & 38 & 0 & 0.6860 & 0.0192 & 1938 & $5.9e-06^{***}$ & $2.3e-05^{***}$ & $1.3e-04^{***}$ & 0.4120 \\
LDLT-R & Sandwich & 121 & 40 & 81 & 0 & 0.3306 & -0.0197 & 2062 & $2.5e-05^{***}$ & $7.6e-05^{***}$ & $5.2e-04^{***}$ & 0.3828 \\
LDLT-R & Orthogonal & 121 & 78 & 43 & 0 & 0.6446 & 0.0178 & 2306 & $3.4e-04^{***}$ & $6.9e-04^{***}$ & $4.5e-03^{**}$ & 0.3254 \\
Orthogonal & SLL & 121 & 66 & 55 & 0 & 0.5455 & 0.0041 & 3178 & $1.9e-01$ & $1.9e-01$ & $8.7e-01$ & 0.1204 \\
\bottomrule
\end{tabular}
\end{adjustbox}
\endgroup
\begin{tablenotes}\item Stars mark significance ($^*\,p\!\le\!0.05$, $^{**}\,p\!\le\!0.01$, $^{***}\,p\!\le\!0.001$).
\end{tablenotes}
\end{threeparttable}
\caption[Mean Certified Accuracy (72/255)]{Wilcoxon signed-rank tests (two-sided) for Mean Certified Accuracy (72/255); $p$-values with Holm FWER corrections within-metric and global.}
\label{tab:wilcoxon:mean_cert_acc_72}
\end{table}

\begin{table}[htb]
\centering
\begin{threeparttable}
\begingroup
\setlength{\tabcolsep}{3pt}
\footnotesize
\begin{adjustbox}{max width=\linewidth}
\begin{tabular}{@{}llrrrrrrrrrrr@{}}
\toprule
\multicolumn{2}{c}{Algorithms} & \multicolumn{6}{c}{Run Statistics} & \multicolumn{5}{c}{Wilcoxon pairwise Statistics} \\\cmidrule(lr){1-2} \cmidrule(lr){3-8} \cmidrule(lr){9-13}
Alg A & Alg B & $n$ & wins$_A$ & wins$_B$ & ties & WinRate A & \shortstack{Median \\ $\Delta$ (A---B)} & $W$ & $p$ & $p_{\text{Holm,within}}$ & $p_{\text{Holm,global}}$ & $r$ \\
\midrule
AOL & LDLT-R & 121 & 1 & 120 & 0 & 0.0083 & -0.2229 & 11 & $1.8e-21^{***}$ & $2.7e-20^{***}$ & $1.3e-19^{***}$ & 0.8651 \\
AOL & SLL & 121 & 4 & 117 & 0 & 0.0331 & -0.1616 & 33 & $3.1e-21^{***}$ & $4.3e-20^{***}$ & $2.3e-19^{***}$ & 0.8599 \\
AOL & Orthogonal & 121 & 6 & 115 & 0 & 0.0496 & -0.1829 & 54 & $5.2e-21^{***}$ & $6.8e-20^{***}$ & $3.8e-19^{***}$ & 0.8550 \\
AOL & Sandwich & 121 & 5 & 116 & 0 & 0.0413 & -0.2530 & 63 & $6.5e-21^{***}$ & $7.8e-20^{***}$ & $4.7e-19^{***}$ & 0.8529 \\
AOL & LDLT-L & 121 & 15 & 102 & 4 & 0.1405 & -0.0795 & 265 & $4.5e-18^{***}$ & $5.0e-17^{***}$ & $2.7e-16^{***}$ & 0.8011 \\
LDLT-L & Sandwich & 121 & 16 & 104 & 1 & 0.1364 & -0.0911 & 469 & $1.3e-16^{***}$ & $0^{***}$ & $0^{***}$ & 0.7556 \\
Orthogonal & Sandwich & 121 & 15 & 106 & 0 & 0.1240 & -0.0411 & 634 & $0^{***}$ & $0^{***}$ & $0^{***}$ & 0.7186 \\
Sandwich & SLL & 121 & 95 & 25 & 1 & 0.7893 & 0.0391 & 769 & $0^{***}$ & $0^{***}$ & $0^{***}$ & 0.6839 \\
LDLT-L & LDLT-R & 121 & 28 & 93 & 0 & 0.2314 & -0.0559 & 854 & $0^{***}$ & $0^{***}$ & $0^{***}$ & 0.6669 \\
LDLT-L & Orthogonal & 121 & 39 & 82 & 0 & 0.3223 & -0.0289 & 1595 & $6.0e-08^{***}$ & $3.6e-07^{***}$ & $1.9e-06^{***}$ & 0.4926 \\
LDLT-L & SLL & 121 & 37 & 82 & 2 & 0.3140 & -0.0231 & 1675 & $5.1e-07^{***}$ & $2.5e-06^{***}$ & $1.4e-05^{***}$ & 0.4605 \\
LDLT-R & SLL & 121 & 83 & 38 & 0 & 0.6860 & 0.0187 & 2059 & $2.5e-05^{***}$ & $9.8e-05^{***}$ & $5.2e-04^{***}$ & 0.3835 \\
LDLT-R & Sandwich & 121 & 43 & 78 & 0 & 0.3554 & -0.0213 & 2138 & $6.0e-05^{***}$ & $1.8e-04^{***}$ & $1.0e-03^{**}$ & 0.3649 \\
LDLT-R & Orthogonal & 121 & 77 & 44 & 0 & 0.6364 & 0.0182 & 2506 & $2.2e-03^{**}$ & $4.4e-03^{**}$ & $2.0e-02^{*}$ & 0.2784 \\
Orthogonal & SLL & 121 & 65 & 56 & 0 & 0.5372 & 0.0038 & 3069 & $1.1e-01$ & $1.1e-01$ & $8.7e-01$ & 0.1460 \\
\bottomrule
\end{tabular}
\end{adjustbox}
\endgroup
\begin{tablenotes}\item Stars mark significance ($^*\,p\!\le\!0.05$, $^{**}\,p\!\le\!0.01$, $^{***}\,p\!\le\!0.001$).
\end{tablenotes}
\end{threeparttable}
\caption[Mean Certified Accuracy (108/255)]{Wilcoxon signed-rank tests (two-sided) for Mean Certified Accuracy (108/255); $p$-values with Holm FWER corrections within-metric and global.}
\label{tab:wilcoxon:mean_cert_acc_108}
\end{table}

\begin{table}[htb]
\centering
\begin{threeparttable}
\begingroup
\setlength{\tabcolsep}{3pt}
\footnotesize
\begin{adjustbox}{max width=\linewidth}
\begin{tabular}{@{}llrrrrrrrrrrr@{}}
\toprule
\multicolumn{2}{c}{Algorithms} & \multicolumn{6}{c}{Run Statistics} & \multicolumn{5}{c}{Wilcoxon pairwise Statistics} \\\cmidrule(lr){1-2} \cmidrule(lr){3-8} \cmidrule(lr){9-13}
Alg A & Alg B & $n$ & wins$_A$ & wins$_B$ & ties & WinRate A & \shortstack{Median \\ $\Delta$ (A---B)} & $W$ & $p$ & $p_{\text{Holm,within}}$ & $p_{\text{Holm,global}}$ & $r$ \\
\midrule
AOL & LDLT-R & 121 & 1 & 111 & 9 & 0.0455 & -0.0570 & 10 & $5.4e-20^{***}$ & $8.2e-19^{***}$ & $3.6e-18^{***}$ & 0.8651 \\
AOL & SLL & 121 & 3 & 104 & 14 & 0.0826 & -0.0476 & 23 & $5.3e-19^{***}$ & $6.9e-18^{***}$ & $3.2e-17^{***}$ & 0.8610 \\
AOL & Sandwich & 121 & 2 & 112 & 7 & 0.0455 & -0.0803 & 62 & $9.9e-20^{***}$ & $1.4e-18^{***}$ & $6.3e-18^{***}$ & 0.8514 \\
AOL & Orthogonal & 121 & 5 & 105 & 11 & 0.0868 & -0.0363 & 122 & $2.4e-18^{***}$ & $2.8e-17^{***}$ & $1.4e-16^{***}$ & 0.8331 \\
AOL & LDLT-L & 121 & 8 & 85 & 28 & 0.1818 & -0.0160 & 182 & $0^{***}$ & $0^{***}$ & $0^{***}$ & 0.7958 \\
Orthogonal & Sandwich & 121 & 11 & 103 & 7 & 0.1198 & -0.0313 & 258 & $1.4e-17^{***}$ & $1.5e-16^{***}$ & $8.1e-16^{***}$ & 0.7995 \\
LDLT-L & Sandwich & 121 & 12 & 102 & 7 & 0.1281 & -0.0397 & 342 & $1.1e-16^{***}$ & $0^{***}$ & $0^{***}$ & 0.7772 \\
LDLT-L & LDLT-R & 121 & 23 & 89 & 9 & 0.2273 & -0.0242 & 927 & $1.0e-10^{***}$ & $6.0e-10^{***}$ & $3.0e-09^{***}$ & 0.6135 \\
Sandwich & SLL & 121 & 90 & 25 & 6 & 0.7686 & 0.0240 & 977 & $0^{***}$ & $4.0e-10^{***}$ & $1.7e-09^{***}$ & 0.6135 \\
LDLT-L & SLL & 121 & 27 & 79 & 15 & 0.2851 & -0.0091 & 1323 & $1.9e-06^{***}$ & $9.4e-06^{***}$ & $4.5e-05^{***}$ & 0.4629 \\
LDLT-R & SLL & 121 & 80 & 35 & 6 & 0.6860 & 0.0175 & 1603 & $1.4e-06^{***}$ & $8.1e-06^{***}$ & $3.4e-05^{***}$ & 0.4506 \\
LDLT-L & Orthogonal & 121 & 36 & 75 & 10 & 0.3388 & -0.0085 & 1741 & $5.8e-05^{***}$ & $2.3e-04^{***}$ & $1.0e-03^{**}$ & 0.3816 \\
LDLT-R & Orthogonal & 121 & 77 & 40 & 4 & 0.6529 & 0.0132 & 1982 & $6.5e-05^{***}$ & $2.3e-04^{***}$ & $1.0e-03^{**}$ & 0.3694 \\
LDLT-R & Sandwich & 121 & 45 & 73 & 3 & 0.3843 & -0.0078 & 2234 & $6.1e-04^{***}$ & $1.2e-03^{**}$ & $7.2e-03^{**}$ & 0.3154 \\
Orthogonal & SLL & 121 & 55 & 58 & 8 & 0.4876 & 0.0000 & 3068 & $6.6e-01$ & $6.6e-01$ & $1.0e+00$ & 0.0410 \\
\bottomrule
\end{tabular}
\end{adjustbox}
\endgroup
\begin{tablenotes}\item Stars mark significance ($^*\,p\!\le\!0.05$, $^{**}\,p\!\le\!0.01$, $^{***}\,p\!\le\!0.001$).
\end{tablenotes}
\end{threeparttable}
\caption[Mean Certified Accuracy (255/255)]{Wilcoxon signed-rank tests (two-sided) for Mean Certified Accuracy (255/255); $p$-values with Holm FWER corrections within-metric and global.}
\label{tab:wilcoxon:mean_cert_acc_255}
\end{table}

\vskip 0.2in
\bibliography{main}

@article{Araujo2023,
   author = {Alexandre Araujo and Aaron Havens and Blaise Delattre and Alexandre Allauzen and Bin Hu},
   month = {3},
   title = {A Unified Algebraic Perspective on Lipschitz Neural Networks},
   url = {http://arxiv.org/abs/2303.03169},
   year = {2023},
}

@misc{Fazlyab2019,
   author = {Mahyar Fazlyab and Alexander Robey and Hamed Hassani and Manfred Morari and George J Pappas},
   title = {Efficient and Accurate Estimation of Lipschitz Constants for Deep Neural Networks},
   year = {2019},
}

@article{Prach2022,
   author = {Bernd Prach and Christoph H. Lampert},
   month = {8},
   title = {Almost-Orthogonal Layers for Efficient General-Purpose Lipschitz Networks},
   url = {https://arxiv.org/abs/2208.03160v2},
   year = {2022},
}

@article{Miyato2018,
   author = {Takeru Miyato and Toshiki Kataoka and Masanori Koyama and Yuichi Yoshida},
   isbn = {1802.05957v1},
   journal = {6th International Conference on Learning Representations, ICLR 2018 - Conference Track Proceedings},
   month = {2},
   publisher = {International Conference on Learning Representations, ICLR},
   title = {Spectral Normalization for Generative Adversarial Networks},
   url = {https://arxiv.org/abs/1802.05957v1},
   year = {2018},
}

@misc{Meunier2022,
   author = {Laurent Meunier and Blaise J Delattre and Alexandre Araujo and Alexandre Allauzen},
   issn = {2640-3498},
   month = {6},
   pages = {15484-15500},
   publisher = {PMLR},
   title = {A Dynamical System Perspective for Lipschitz Neural Networks},
   url = {https://proceedings.mlr.press/v162/meunier22a.html},
   year = {2022},
}

@article{Bartlett2017,
   author = {Peter L. Bartlett and Dylan J. Foster and Matus Telgarsky},
   issn = {10495258},
   journal = {Advances in Neural Information Processing Systems},
   month = {6},
   pages = {6241-6250},
   publisher = {Neural information processing systems foundation},
   title = {Spectrally-normalized margin bounds for neural networks},
   volume = {2017-December},
   url = {https://arxiv.org/abs/1706.08498v2},
   year = {2017},
}

@article{Klambauer2017,
   author = {Günter Klambauer and Thomas Unterthiner and Andreas Mayr and Sepp Hochreiter},
   journal = {Advances in Neural Information Processing Systems},
   title = {Self-Normalizing Neural Networks},
   volume = {30},
   year = {2017},
}

@article{Szegedy2016,
   author = {Christian Szegedy and Sergey Ioffe and Vincent Vanhoucke and Alexander A. Alemi},
   doi = {10.1609/aaai.v31i1.11231},
   issn = {2159-5399},
   journal = {31st AAAI Conference on Artificial Intelligence, AAAI 2017},
   month = {2},
   pages = {4278-4284},
   publisher = {AAAI press},
   title = {Inception-v4, Inception-ResNet and the Impact of Residual Connections on Learning},
   url = {https://arxiv.org/abs/1602.07261v2},
   year = {2016},
}

@article{Clevert2015,
   author = {Djork-Arné Clevert and Thomas Unterthiner and Sepp Hochreiter},
   journal = {4th International Conference on Learning Representations, ICLR 2016 - Conference Track Proceedings},
   month = {11},
   publisher = {International Conference on Learning Representations, ICLR},
   title = {Fast and Accurate Deep Network Learning by Exponential Linear Units (ELUs)},
   url = {http://arxiv.org/abs/1511.07289},
   year = {2015},
}

@article{Cancino2002,
   author = {Hector F. Cancino‐De‐Greiff and Rubersy Ramos‐Garcia and Juan V. Lorenzo‐Ginori},
   doi = {10.1002/cmr.10043},
   issn = {1043-7347},
   issue = {6},
   journal = {Concepts in Magnetic Resonance},
   month = {1},
   pages = {388-401},
   title = {Signal de‐noising in magnetic resonance spectroscopy using wavelet transforms},
   volume = {14},
   url = {https://onlinelibrary.wiley.com/doi/10.1002/cmr.10043},
   year = {2002},
}

@article{Courbariaux2015,
   author = {Matthieu Courbariaux and Yoshua Bengio and Jean Pierre David},
   issn = {10495258},
   journal = {Advances in Neural Information Processing Systems},
   month = {11},
   pages = {3123-3131},
   publisher = {Neural information processing systems foundation},
   title = {BinaryConnect: Training Deep Neural Networks with binary weights during propagations},
   volume = {2015-January},
   url = {https://arxiv.org/abs/1511.00363v3},
   year = {2015},
}

@PHDTHESIS{collobert2004,
         author = {Collobert, Ronan},
       projects = {Idiap},
          title = {Large Scale Machine Learning},
           year = {2004},
         school = {Universit{\'{e}} de Paris VI},
     postscript = {ftp://ftp.idiap.ch/pub/reports/2004/collobert_2004_phdthesis.ps.gz},
ipdmembership={learning},
}

@article{Howard2019,
   author = {Andrew Howard and Mark Sandler and Bo Chen and Weijun Wang and Liang Chieh Chen and Mingxing Tan and Grace Chu and Vijay Vasudevan and Yukun Zhu and Ruoming Pang and Quoc Le and Hartwig Adam},
   doi = {10.1109/ICCV.2019.00140},
   isbn = {9781728148038},
   issn = {15505499},
   journal = {Proceedings of the IEEE International Conference on Computer Vision},
   month = {5},
   pages = {1314-1324},
   publisher = {Institute of Electrical and Electronics Engineers Inc.},
   title = {Searching for MobileNetV3},
   volume = {2019-October},
   url = {https://arxiv.org/abs/1905.02244v5},
   year = {2019},
}

@article{McCulloch1943,
   author = {Warren S. McCulloch and Walter Pitts},
   doi = {10.1007/BF02478259/METRICS},
   issn = {00074985},
   issue = {4},
   journal = {The Bulletin of Mathematical Biophysics},
   month = {12},
   pages = {115-133},
   publisher = {Kluwer Academic Publishers},
   title = {A logical calculus of the ideas immanent in nervous activity},
   volume = {5},
   url = {https://link.springer.com/article/10.1007/BF02478259},
   year = {1943},
}

@article{Maas2013,
   author = {Andew L. Maas and Y. Hannun Awni and Andrew Y Ng},
   journal = {Proceedings of the 30th International Conference on Machine Learning},
   month = {6},
   title = {Rectifier Nonlinearities Improve Neural Network Acoustic Models},
   volume = {28},
   year = {2013},
}

@article{He2015PReLU,
   author = {Kaiming He and Xiangyu Zhang and Shaoqing Ren and Jian Sun},
   journal = {CoRR},
   month = {2},
   title = {Delving Deep into Rectifiers: Surpassing Human-Level Performance on ImageNet Classification},
   volume = {abs/1502.01852},
   url = {http://arxiv.org/abs/1502.01852},
   year = {2015},
}

@article{Xu2015,
   author = {Bing Xu and Naiyan Wang and Tianqi Chen and Mu Li},
   month = {5},
   title = {Empirical Evaluation of Rectified Activations in Convolutional Network},
   url = {http://arxiv.org/abs/1505.00853},
   year = {2015},
}

@article{Howard2017,
   author = {Andrew G. Howard and Menglong Zhu and Bo Chen and Dmitry Kalenichenko and Weijun Wang and Tobias Weyand and Marco Andreetto and Hartwig Adam},
   month = {4},
   title = {MobileNets: Efficient Convolutional Neural Networks for Mobile Vision Applications},
   url = {http://arxiv.org/abs/1704.04861},
   year = {2017},
}

@article{Barron2017,
   author = {Jonathan T. Barron},
   month = {4},
   title = {Continuously Differentiable Exponential Linear Units},
   url = {http://arxiv.org/abs/1704.07483},
   year = {2017},
}

@article{Hendrycks2016,
   author = {Dan Hendrycks and Kevin Gimpel},
   month = {6},
   title = {Gaussian Error Linear Units (GELUs)},
   url = {https://arxiv.org/abs/1606.08415v5},
   year = {2016},
}

@article{Sak2014,
   author = {Haşim Sak and Andrew Senior and Françoise Beaufays},
   month = {2},
   title = {Long Short-Term Memory Based Recurrent Neural Network Architectures for Large Vocabulary Speech Recognition},
   url = {https://arxiv.org/abs/1402.1128v1},
   year = {2014},
}

@article{Elfwing2017,
   author = {Stefan Elfwing and Eiji Uchibe and Kenji Doya},
   doi = {10.1016/j.neunet.2017.12.012},
   issn = {18792782},
   journal = {Neural Networks},
   month = {2},
   pages = {3-11},
   pmid = {29395652},
   publisher = {Elsevier Ltd},
   title = {Sigmoid-Weighted Linear Units for Neural Network Function Approximation in Reinforcement Learning},
   volume = {107},
   url = {https://arxiv.org/abs/1702.03118v3},
   year = {2017},
}

@article{Zhou2016,
   author = {Mingyuan Zhou},
   month = {8},
   title = {Softplus Regressions and Convex Polytopes},
   url = {http://arxiv.org/abs/1608.06383},
   year = {2016},
}

@article{Misra2019,
   author = {Diganta Misra},
   journal = {31st British Machine Vision Conference, BMVC 2020},
   month = {8},
   publisher = {British Machine Vision Association, BMVA},
   title = {Mish: A Self Regularized Non-Monotonic Activation Function},
   url = {http://arxiv.org/abs/1908.08681},
   year = {2019},
}

@article{Ping2017,
   author = {Wei Ping and Kainan Peng and Andrew Gibiansky and Sercan O. Arik and Ajay Kannan and Sharan Narang and Jonathan Raiman and John Miller},
   journal = {6th International Conference on Learning Representations, ICLR 2018 - Conference Track Proceedings},
   month = {10},
   publisher = {International Conference on Learning Representations, ICLR},
   title = {Deep Voice 3: Scaling Text-to-Speech with Convolutional Sequence Learning},
   url = {http://arxiv.org/abs/1710.07654},
   year = {2017},
}

@article{Xu2024,
   author = {Yuezhu Xu and S. Sivaranjani},
   month = {4},
   title = {ECLipsE: Efficient Compositional Lipschitz Constant Estimation for Deep Neural Networks},
   url = {https://arxiv.org/abs/2404.04375v2},
   year = {2024},
}

@article{He2015ResNet,
   author = {Kaiming He and Xiangyu Zhang and Shaoqing Ren and Jian Sun},
   doi = {10.1109/CVPR.2016.90},
   isbn = {9781467388504},
   issn = {10636919},
   journal = {Proceedings of the IEEE Computer Society Conference on Computer Vision and Pattern Recognition},
   month = {12},
   pages = {770-778},
   publisher = {IEEE Computer Society},
   title = {Deep Residual Learning for Image Recognition},
   volume = {2016-December},
   url = {https://arxiv.org/abs/1512.03385v1},
   year = {2015},
}

@article{Zagoruyko2016,
   author = {Sergey Zagoruyko and Nikos Komodakis},
   doi = {10.5244/C.30.87},
   journal = {British Machine Vision Conference 2016, BMVC 2016},
   month = {5},
   pages = {87.1-87.12},
   publisher = {British Machine Vision Conference, BMVC},
   title = {Wide Residual Networks},
   volume = {2016-September},
   url = {https://arxiv.org/abs/1605.07146v4},
   year = {2016},
}

@article{Hu2017,
   author = {Jie Hu and Li Shen and Samuel Albanie and Gang Sun and Enhua Wu},
   doi = {10.1109/TPAMI.2019.2913372},
   issn = {19393539},
   issue = {8},
   journal = {IEEE Transactions on Pattern Analysis and Machine Intelligence},
   month = {9},
   pages = {2011-2023},
   pmid = {31034408},
   publisher = {IEEE Computer Society},
   title = {Squeeze-and-Excitation Networks},
   volume = {42},
   url = {https://arxiv.org/abs/1709.01507v4},
   year = {2017},
}

@article{Xie2016,
   author = {Saining Xie and Ross Girshick and Piotr Dollár and Zhuowen Tu and Kaiming He},
   doi = {10.1109/CVPR.2017.634},
   isbn = {9781538604571},
   journal = {Proceedings - 30th IEEE Conference on Computer Vision and Pattern Recognition, CVPR 2017},
   month = {11},
   pages = {5987-5995},
   publisher = {Institute of Electrical and Electronics Engineers Inc.},
   title = {Aggregated Residual Transformations for Deep Neural Networks},
   volume = {2017-January},
   url = {https://arxiv.org/abs/1611.05431v2},
   year = {2016},
}

@inbook{10.5555/3454287.3455312,
author = {Fazlyab, Mahyar and Robey, Alexander and Hassani, Hamed and Morari, Manfred and Pappas, George J.},
title = {Efficient and accurate estimation of lipschitz constants for deep neural networks},
year = {2019},
publisher = {Curran Associates Inc.},
address = {Red Hook, NY, USA},
booktitle = {Proceedings of the 33rd International Conference on Neural Information Processing Systems},
articleno = {1025},
numpages = {12}
}

@article{Tsuzuku2018,
  author       = {Yusuke Tsuzuku and
                  Issei Sato and
                  Masashi Sugiyama},
  title        = {Lipschitz-Margin Training: Scalable Certification of Perturbation
                  Invariance for Deep Neural Networks},
  journal      = {CoRR},
  volume       = {abs/1802.04034},
  year         = {2018},
  url          = {http://arxiv.org/abs/1802.04034},
  eprinttype    = {arXiv},
  eprint       = {1802.04034},
  bibsource    = {dblp computer science bibliography, https://dblp.org}
}

@article{Inkawhich2019,
   author = {Matthew Inkawhich and Yiran Chen and Hai Li},
   isbn = {9781450375184},
   issn = {15582914},
   journal = {Proceedings of the International Joint Conference on Autonomous Agents and Multiagent Systems, AAMAS},
   month = {5},
   pages = {557-565},
   publisher = {International Foundation for Autonomous Agents and Multiagent Systems (IFAAMAS)},
   title = {Snooping Attacks on Deep Reinforcement Learning},
   volume = {2020-May},
   url = {https://arxiv.org/abs/1905.11832v2},
   year = {2019},
}

@article{Goodfellow2014,
   author = {Ian J. Goodfellow and Jonathon Shlens and Christian Szegedy},
   journal = {3rd International Conference on Learning Representations, ICLR 2015 - Conference Track Proceedings},
   month = {12},
   publisher = {International Conference on Learning Representations, ICLR},
   title = {Explaining and Harnessing Adversarial Examples},
   url = {https://arxiv.org/abs/1412.6572v3},
   year = {2014},
}

@article{Sandryhaila2013,
   author = {Aliaksei Sandryhaila and Jose M. F. Moura},
   month = {6},
   title = {Eigendecomposition of Block Tridiagonal Matrices},
   url = {https://arxiv.org/abs/1306.0217v1},
   year = {2013},
}

@article{Agarwal2019,
   author = {Etika Agarwal and S. Sivaranjani and Vijay Gupta and Panos Antsaklis},
   doi = {10.23919/ACC.2019.8814701},
   isbn = {9781538679265},
   issn = {07431619},
   journal = {Proceedings of the American Control Conference},
   month = {7},
   pages = {5816-5821},
   publisher = {Institute of Electrical and Electronics Engineers Inc.},
   title = {Sequential synthesis of distributed controllers for cascade interconnected systems},
   volume = {2019-July},
   year = {2019},
}

@article{Meunier2021,
   author = {Laurent Meunier and Blaise Delattre and Alexandre Araujo and Alexandre Allauzen},
   issn = {26403498},
   journal = {Proceedings of Machine Learning Research},
   month = {10},
   pages = {15484-15500},
   publisher = {ML Research Press},
   title = {A Dynamical System Perspective for Lipschitz Neural Networks},
   volume = {162},
   url = {https://arxiv.org/abs/2110.12690v2},
   year = {2021},
}

@article{Gouk2021,
   author = {Henry Gouk and Eibe Frank and Bernhard Pfahringer and Michael J. Cree},
   doi = {10.1007/s10994-020-05929-w},
   issn = {0885-6125},
   issue = {2},
   journal = {Machine Learning},
   month = {2},
   pages = {393-416},
   title = {Regularisation of neural networks by enforcing Lipschitz continuity},
   volume = {110},
   url = {http://link.springer.com/10.1007/s10994-020-05929-w},
   year = {2021},
}

@article{Aziznejad2020,
   author = {Shayan Aziznejad and Harshit Gupta and Joaquim Campos and Michael Unser},
   doi = {10.1109/TSP.2020.3014611},
   journal = {IEEE Transactions on Signal Processing},
   month = {1},
   pages = {4688-4699},
   publisher = {Institute of Electrical and Electronics Engineers Inc.},
   title = {Deep Neural Networks with Trainable Activations and Controlled Lipschitz Constant},
   volume = {68},
   url = {http://arxiv.org/abs/2001.06263 http://dx.doi.org/10.1109/TSP.2020.3014611},
   year = {2020},
}

@article{Bear2024,
   author = {Jay Bear and Adam Prügel-Bennett and Jonathon Hare},
   month = {10},
   title = {Rethinking Deep Thinking: Stable Learning of Algorithms using Lipschitz Constraints},
   url = {https://arxiv.org/abs/2410.23451v1},
   year = {2024},
}

@inbook{Watkins2002,
   author = {David S. Watkins},
   doi = {10.1002/0471249718.ch1},
   isbn = {9780471249719},
   booktitle = {Fundamentals of Matrix Computations},
   month = {5},
   note = {Wiley Online Books},
   pages = {1-110},
   publisher = {Wiley},
   title = {Gaussian Elimination and Its Variants},
   url = {https://onlinelibrary.wiley.com/doi/10.1002/0471249718.ch1},
   year = {2002}
}

@article{Higham1986NewtonsRoot,
    title = {{Newton's Method for the Matrix Square Root}},
    year = {1986},
    journal = {Mathematics of Computation},
    author = {Higham, Nicholas J.},
    number = {174},
    month = {4},
    pages = {537},
    volume = {46},
    publisher = {JSTOR},
    doi = {10.2307/2007992},
    issn = {00255718}
}

@article{Horn2012MatrixAnalysis,
    title = {{Matrix Analysis}},
    year = {2012},
    journal = {Matrix Analysis},
    author = {Horn, Roger A. and Johnson, Charles R.},
    month = {10},
    publisher = {Cambridge University Press},
    isbn = {9781139020411},
    doi = {10.1017/CBO9781139020411}
}

@article{Pan1999ComplexityEigenproblem,
    title = {{Complexity of the matrix eigenproblem}},
    year = {1999},
    journal = {Conference Proceedings of the Annual ACM Symposium on Theory of Computing},
    author = {Pan, Victor Y. and Chen, Zhao Q.},
    pages = {507--516},
    publisher = {ACM},
    url = {https://dl.acm.org/doi/pdf/10.1145/301250.301389},
    doi = {10.1145/301250.301389/ASSET/FC0FB8FD-669F-4EF9-B898-30A3DF5DA2BA/ASSETS/301250.301389.FP.PNG},
    issn = {07349025}
}

@misc{CholeskyLibrary,
    title = {{Cholesky decomposition - C++, C{\#}, Java library}},
    url = {https://www.alglib.net/matrixops/cholesky.php},
    year = {1999},
    author = {ALGLIB}
}

@misc{Taboga2021TriangularMatrix,
    title = {{Triangular matrix}},
    year = {2021},
    booktitle = {Lectures on matrix algebra},
    author = {Taboga, Marco},
    url = {https://www.statlect.com/matrix-algebra/triangular-matrix}
}

@book{Woodbury1950InvertingMatrices,
    title = {{Inverting modified matrices}},
    year = {1950},
    author = {Woodbury, Max A},
    pages = {4},
    publisher = {Princeton University, Princeton, NJ}
}

@article{Gray2005ToeplitzReview,
    title = {{Toeplitz and Circulant Matrices: A Review}},
    year = {2005},
    journal = {Foundations and Trends{\textregistered} in Communications and Information Theory},
    author = {Gray, Robert M.},
    number = {3},
    pages = {155--239},
    volume = {2},
    url = {http://www.nowpublishers.com/article/Details/CIT-006},
    doi = {10.1561/0100000006},
    issn = {1567-2190}
}

@book{Rao2018TheHandbook,
    title = {{The Transform and Data Compression Handbook}},
    year = {2018},
    booktitle = {The Transform and Data Compression Handbook},
    author = {Rao, Ed K R and Yip, P C},
    editor = {Rao, Kamisetty Ramam and Yip, Patrick C.},
    month = {10},
    publisher = {CRC Press},
    url = {https://www.taylorfrancis.com/books/9781420037388},
    isbn = {9781315220529},
    doi = {10.1201/9781315220529}
}

@article{Singla2021ImprovedCIFAR-100,
    title = {{Improved deterministic l2 robustness on CIFAR-10 and CIFAR-100}},
    year = {2021},
    journal = {ICLR 2022 - 10th International Conference on Learning Representations},
    author = {Singla, Sahil and Singla, Surbhi and Feizi, Soheil},
    month = {8},
    publisher = {International Conference on Learning Representations, ICLR},
    url = {https://arxiv.org/pdf/2108.04062},
    arxivId = {2108.04062}
}

@article{Meunier2021ANetworks,
    title = {{A Dynamical System Perspective for Lipschitz Neural Networks}},
    year = {2021},
    journal = {Proceedings of Machine Learning Research},
    author = {Meunier, Laurent and Delattre, Blaise and Araujo, Alexandre and Allauzen, Alexandre},
    month = {10},
    pages = {15484--15500},
    volume = {162},
    publisher = {ML Research Press},
    url = {https://arxiv.org/abs/2110.12690v2},
    issn = {26403498},
    arxivId = {2110.12690}
}

@article{Juston20251-LipschitzProblem,
    title = {{1-Lipschitz Network Initialization for Certifiably Robust Classification Applications: A Decay Problem}},
    year = {2025},
    author = {Juston, Marius F. R. and Norris, William R. and Nottage, Dustin and Soylemezoglu, Ahmet},
    month = {2},
    url = {https://arxiv.org/pdf/2503.00240},
    arxivId = {2503.00240}
}

@article{Ronneberger2015U-Net:Segmentation,
    title = {{U-Net: Convolutional Networks for Biomedical Image Segmentation}},
    year = {2015},
    journal = {Lecture Notes in Computer Science (including subseries Lecture Notes in Artificial Intelligence and Lecture Notes in Bioinformatics)},
    author = {Ronneberger, Olaf and Fischer, Philipp and Brox, Thomas},
    pages = {234--241},
    volume = {9351},
    publisher = {Springer, Cham},
    url = {https://link.springer.com/chapter/10.1007/978-3-319-24574-4_28},
    isbn = {978-3-319-24574-4},
    doi = {10.1007/978-3-319-24574-4{\_}28},
    issn = {1611-3349}
}

@article{Rombach2021High-ResolutionModels,
    title = {{High-Resolution Image Synthesis with Latent Diffusion Models}},
    year = {2021},
    journal = {Proceedings of the IEEE Computer Society Conference on Computer Vision and Pattern Recognition},
    author = {Rombach, Robin and Blattmann, Andreas and Lorenz, Dominik and Esser, Patrick and Ommer, Bjorn},
    month = {12},
    pages = {10674--10685},
    volume = {2022-June},
    publisher = {IEEE Computer Society},
    url = {https://arxiv.org/pdf/2112.10752},
    isbn = {9781665469463},
    doi = {10.1109/CVPR52688.2022.01042},
    issn = {10636919},
    arxivId = {2112.10752}
}

@article{Kingma2019AnAutoencoders,
    title = {{An Introduction to Variational Autoencoders}},
    year = {2019},
    journal = {Foundations and Trends in Machine Learning},
    author = {Kingma, Diederik P. and Welling, Max},
    number = {4},
    month = {6},
    pages = {307--392},
    volume = {12},
    publisher = {Now Publishers Inc},
    url = {https://arxiv.org/pdf/1906.02691},
    doi = {10.1561/2200000056},
    issn = {19358245},
    arxivId = {1906.02691}
}

@article{Lessard2015AnalysisConstraints,
    title = {{Analysis and Design of Optimization Algorithms via Integral Quadratic Constraints}},
    year = {2015},
    journal = {SIAM Journal on Optimization},
    author = {Lessard, Laurent and Recht, Benjamin and Packard, Andrew},
    number = {1},
    month = {10},
    pages = {57--95},
    volume = {26},
    publisher = {Society for Industrial and Applied Mathematics Publications},
    url = {http://arxiv.org/abs/1408.3595 http://dx.doi.org/10.1137/15M1009597},
    doi = {10.1137/15M1009597},
    arxivId = {1408.3595v7}
}

@article{Megretski1997SystemConstraints,
    title = {{System analysis via integral quadratic constraints}},
    year = {1997},
    journal = {IEEE Transactions on Automatic Control},
    author = {Megretski, Alexandre and Rantzer, Anders},
    number = {6},
    pages = {819--830},
    volume = {42},
    doi = {10.1109/9.587335},
    issn = {00189286}
}

@article{Wang2023DirectNetworks,
    title = {{Direct Parameterization of Lipschitz-Bounded Deep Networks}},
    year = {2023},
    journal = {Proceedings of Machine Learning Research},
    author = {Wang, Ruigang and Manchester, Ian R.},
    month = {1},
    pages = {36093--36110},
    volume = {202},
    publisher = {ML Research Press},
    url = {https://arxiv.org/pdf/2301.11526},
    issn = {26403498},
    arxivId = {2301.11526}
}

@article{Winston2020MonotoneNetworks,
    title = {{Monotone operator equilibrium networks}},
    year = {2020},
    journal = {Advances in Neural Information Processing Systems},
    author = {Winston, Ezra and Zico Kolter, J.},
    month = {6},
    volume = {2020-December},
    publisher = {Neural information processing systems foundation},
    url = {https://arxiv.org/pdf/2006.08591},
    issn = {10495258},
    arxivId = {2006.08591}
}

@article{Havens2023ExploitingModels,
    title = {{Exploiting Connections between Lipschitz Structures for Certifiably Robust Deep Equilibrium Models}},
    year = {2023},
    journal = {Advances in Neural Information Processing Systems},
    author = {Havens, Aaron J and Araujo, Alexandre and Garg, Siddharth and Khorrami, Farshad and Hu, Bin},
    month = {12},
    pages = {21658--21674},
    volume = {36},
    url = {https://github.com/AaronHavens/ExploitingLipschitzDEQ.}
}

@article{Fernandez-Delgado2014DoProblems,
    title = {{Do we Need Hundreds of Classifiers to Solve Real World Classification Problems?}},
    year = {2014},
    journal = {Journal of Machine Learning Research},
    author = {Fern{\'{a}}ndez-Delgado, Manuel and Cernadas, Eva and Barro, Senén and Amorim, Dinani and Fern{\'{a}}ndez-Delgado, Amorim},
    number = {90},
    pages = {3133--3181},
    volume = {15},
    url = {http://jmlr.org/papers/v15/delgado14a.html},
    issn = {1533-7928}
}

@article{Demsar2006StatisticalSets,
    title = {{Statistical Comparisons of Classifiers over Multiple Data Sets}},
    year = {2006},
    journal = {Journal of Machine Learning Research},
    author = {Dem{\v{s}}ar, Janez},
    number = {1},
    pages = {1--30},
    volume = {7},
    url = {http://jmlr.org/papers/v7/demsar06a.html},
    issn = {1533-7928}
}

\end{document}